\pdfoutput=1
\documentclass{article}

\usepackage[final]{neurips_2021} 

\usepackage[utf8]{inputenc} 
\usepackage[T1]{fontenc}    
\usepackage{url}            
\usepackage{booktabs}       
\usepackage{longtable}
\usepackage{amsfonts}       
\usepackage{nicefrac}       
\usepackage{microtype}      
\usepackage{xcolor}         
\usepackage{mathrsfs}
\usepackage{indentfirst}
\usepackage{multirow}
\usepackage{verbatim}
\usepackage{subfigure}
\usepackage{amsmath}
\usepackage{amssymb,amsthm,mathtools}
\usepackage{graphicx}
\usepackage{url}
\usepackage{authblk}
\usepackage{float}
\usepackage{enumitem}
\definecolor{bluegray}{rgb}{0.4, 0.6, 0.8}
\definecolor{darkblue}{rgb}{0,0.08,0.5}
\definecolor{forestgreen}{rgb}{0.05,0.45,0.05}
\usepackage[colorlinks,
            linkcolor=forestgreen,
            citecolor=darkblue]{hyperref} 
\usepackage[linesnumbered,ruled,vlined]{algorithm2e}
\renewenvironment{proof}{\paragraph{Proof:}}{\hfill$\square$}
\usepackage[nameinlink,noabbrev,capitalise]{cleveref}
\Crefname{assumption}{Assumption}{Assumptions}
\crefname{equation}{}{}
\Crefname{lemma}{Lemma}{Lemmas}
\Crefname{definition}{Definition}{Definitions}
\Crefname{proposition}{Proposition}{Propositions}
\Crefname{corollary}{Corollary}{Corollaries}
\Crefname{example}{Example}{Examples}
\newtheorem{theorem}{Theorem}[section]
\newtheorem{lemma}[theorem]{Lemma}
\newtheorem{assumption}[theorem]{Assumption}
\newtheorem{definition}[theorem]{Definition}
\newtheorem{remark}[theorem]{Remark}
\newtheorem{proposition}[theorem]{Proposition}
\newtheorem{corollary}[theorem]{Corollary}
\newtheorem{example}[theorem]{Example}
\usepackage{titlesec}
\usepackage{footmisc}

\newcommand{\ccref}[1]{\textcolor{forestgreen}{\cref{#1}}}
\newcommand{\Ccref}[1]{\textcolor{forestgreen}{\Cref{#1}}}
\title{Non-convex Distributionally Robust Optimization: Non-asymptotic Analysis}

\newcommand*\samethanks[1][\value{footnote}]{\footnotemark[#1]}
\author{%
    \textbf{Jikai Jin}$^{1,}$\thanks{Equal Contribution, alphabetical order.}\quad \textbf{Bohang Zhang}$^{2,}$\samethanks\quad \textbf{Haiyang Wang}$^{3}$\quad  \textbf{Liwei Wang}$^{2,3,}\thanks{Corresponding author.}$\\
    $^1$School of Mathematical Sciences, Peking University\\
    $^2$Key Laboratory of Machine Perception, MOE, School of EECS, Peking University\\
    $^3$Center of Data Science, Peking University\\
    \texttt{\footnotesize \{jkjin,zhangbohang\}@pku.edu.cn,\quad wanghaiyang6@stu.pku.edu.cn,\quad wanglw@cis.pku.edu.cn}\\
}

\begin{document}

\maketitle

\begin{abstract}
Distributionally robust optimization (DRO) is a widely-used approach to learn models that are robust against distribution shift. Compared with the standard optimization setting, the objective function in DRO is more difficult to optimize, and most of the existing theoretical results make strong assumptions on the loss function. In this work we bridge the gap by studying DRO algorithms for general smooth non-convex losses. By carefully exploiting the specific form of the DRO objective, we are able to provide non-asymptotic convergence guarantees even though the objective function is possibly non-convex, non-smooth and has unbounded gradient noise. In particular, we prove that a special algorithm called the mini-batch normalized gradient descent with momentum, can find an $\epsilon$-first-order stationary point 
within $\mathcal O(\epsilon^{-4})$ gradient complexity. We also discuss the conditional value-at-risk (CVaR) setting, where we propose a penalized DRO objective based on a smoothed version of the CVaR that allows us to obtain a similar convergence guarantee. We finally verify our theoretical results in a number of tasks and find that the proposed algorithm can consistently achieve prominent acceleration.
\end{abstract}

\section{Introduction}
For a classical machine learning problem, the goal is typically to train a model over a training set that achieves good performance on a test set, where both the training set and the test set are drawn from the \textit{same} distribution $P$. While such an assumption is reasonable and simple for theoretical analysis, it is often not the case in real applications. For example, this setting may be improper when there is a gap between training and test distribution (e.g. in domain adaptation tasks) \citep{zhang2021coping}, when there is severe class imbalance in the training set \citep{Sagawa2020Distributionally}, when fairness in minority groups is an important consideration \citep{hashimoto2018fairness}, or when the deployed model is exposed to adversarial attacks \citep{sinha2018certifiable}.

Distributionally robust optimization (DRO), as a popular approach to deal with the above situations, has attracted great interest for the machine learning research communities in recent years. In contrast to classic machine learning problems, for DRO it is desired that the trained model still has good performance under distribution shift. Specifically, DRO proposes to minimize the worst-case loss over a set of probability distributions $Q$ around $P$. This can be formulated as the following constrained optimization problem \citep{rahimian2019distributionally,shapiro2017distributionally}:
\begin{equation}
\label{constrained}
    \text{minimize}_{x \in\mathcal{X}} \quad \Psi(x) := \sup_{Q \in\mathcal{U}(P)} \mathbb{E}_{\xi\sim Q}\left[ \ell(x;\xi) \right] 
\end{equation}
where $x\in\mathcal X$ is the parameter to be optimized, $\xi$ is a sample randomly drawn from distribution $Q$, and $\ell(x;\xi)$ is the loss function so that $\mathbb{E}_{\xi\sim Q}\left[ \ell(x;\xi)\right]$ is the expected loss over distribution $Q$. The DRO objective $\Psi(x)$ is therefore the worst-case loss when the distribution $P$ is shifted to $Q$. The set $\mathcal{U}(P)$ is called the uncertainty set and typically defined as
\begin{equation}
\label{uncertainty_set}
    \mathcal{U}(P)  := \left\{ Q : d(Q,P) \leq \epsilon \right\}
\end{equation}
where $d$ measures the distance between two probability distributions, and the positive number $\epsilon$ corresponds to the magnitude of the uncertainty set. 

Instead of imposing a hard constrained uncertainty set, sometimes it is more preferred to use a soft penalty term, resulting in the penalized DRO problem \citep{sinha2018certifiable}:
\begin{equation}
    \label{DRO}
    \text{minimize}_{x \in\mathcal{X}} \quad \Psi(x) := \sup_{Q} \left\{\mathbb{E}_{\xi\sim Q}\left[ \ell(x;\xi) \right] - \lambda d(Q,P) \right\}
\end{equation}
where $\lambda > 0$ is the regularization coefficient. 

There are many possible choices of $d$. A detailed discussion of different distance measures and their properties can be found in \cite{rahimian2019distributionally}. In this paper we consider a general class of distances $d$ called the $\psi$-divergence, which is a popular choice in DRO literature \citep{namkoong2016stochastic,shapiro2017distributionally}. Specifically, for a non-negative convex function $\psi$ such that $\psi(1)=0$ and two probability distributions $P,Q$ such that $Q$ is absolutely continuous w.r.t. $P$, the $\psi$-divergence between $Q$ and $P$ is defined as
\begin{equation}
\notag
    d_{\psi}(Q,P) := \int \psi\left( \frac{\text{d}Q}{\text{d}P} \right) \text{d}P.
\end{equation}
which satisfies $d_{\psi}(Q,P)\ge 0$ and $d_{\psi}(Q,P)= 0$ if $Q=P$ a.s.

The main focus of this paper is to study efficient first-order optimization algorithms for DRO problem \ccref{DRO} for \textit{non-convex} losses $\ell(x,\xi)$. While non-convex models (especially deep neural networks) have been extensively used in DRO setting (e.g. \cite{Sagawa2020Distributionally}), theoretical analysis about the convergence speed is still lacking. Most previous works (e.g. \cite{levy2020large}) assume the loss $\ell(\cdot,\xi)$ is convex, and in this case \ccref{DRO} is equivalent to a convex optimization problem (see \Ccref{preliminaries} for details). Recently some works provide convergence rates of algorithms for non-convex losses in certain special cases, e.g. the divergence measure $\psi$ is chosen as the conditional-value-at-risk (CVaR) and the loss function has some nice structural properties \citep{soma2020statistical,kalogerias2020noisy}. \citet{gurbuzbalaban2020stochastic} considered a more general setting but only proved an asymptotic convergence result for non-convex DRO. 

Compared with these works, we provide the first \textit{non-asymptotic} analysis of optimization algorithms for DRO with \textit{general smooth non-convex} losses $\ell(x,\xi)$ and general $\psi$-divergence. In this setting, there are two major difficulties we must encounter: $(\mathrm{i})$ the DRO objective $\Psi(x)$ is non-convex and can become arbitrarily \textit{non-smooth}, causing standard techniques in smooth non-convex optimization fail to provide a good convergence guarantee; $(\mathrm{ii})$ the noise of the stochastic gradient of $\Psi(x)$ can be arbitrarily large and unbounded even if we assume the gradient of the inner loss $\ell(x,\xi)$ has bounded variance. To tackle these challenges, we propose to optimize the DRO objective using \textit{mini-batch normalized SGD with momentum}, and we are able to prove an $\mathcal{O}(\epsilon^{-4})$ complexity of this algorithm.
The core technique here is to exploit the specific structure of $\Psi(x)$, which shows that $(\mathrm{i})$ the DRO objective satisfies a generalized smoothness condition \citep{NEURIPS2020_b282d173,zhang2019gradient} and $(\mathrm{ii})$ the variance of the stochastic gradient can be bounded by the true gradient. This motivates us to adopt the special algorithm that combines gradient normalization and momentum techniques into SGD, by which both non-smoothness and unbounded noise can be tackled, finally resulting in an $\mathcal{O}(\epsilon^{-4})$ complexity similar to standard smooth non-convex optimization.

The above analysis applies to a broad class of divergence functions $\psi$. We further discuss special cases when $\psi$ has additional properties. In particular, to handle the CVaR case (a non-differentiable loss), we propose a divergence function which is a smoothed variant of CVaR and is further Lipschitz. In this case we show that a convergence guarantee can be established using vanilla SGD, and an similar complexity bound holds.

We highlight that the algorithm and analysis in this paper are not limited to DRO setting, and are described in the context of a general class of optimization problem. Our analysis clearly demonstrates the effectiveness of gradient normalization and momentum techniques in optimizing ill-conditioned objective functions. We believe our result can shed light on why some popular optimizers, in particular Adam \citep{KingmaB14}, often exhibit superior performance in real applications. 

\paragraph{Contributions.} We summarize our main results and contributions below. Let $\psi^*$ be the conjugate function of $\psi$ (see  \Ccref{def_conjugate}). For non-convex optimization problems, since obtaining the global minima is NP-hard in general, this paper adopts the commonly used (relaxed) criteria: to find an $\epsilon$-approximate first-order stationary point of the function $\Psi$ (see \Ccref{def_stationary_point}). We measure the complexity of optimization algorithms by the number of computations of the stochastic gradient $\nabla \ell(x,\xi)$ to reach an $\epsilon$-stationary point.

\begin{itemize}[topsep=0pt]
    \item Assuming that $\psi^*$ is smooth and the loss $\ell$ is Lipschitz and smooth (possibly non-convex or unbounded), we show in \Ccref{section_snm} that the mini-batch normalized momentum algorithm (cf. \Ccref{SNM}) has a complexity of $\mathcal{O}(\epsilon^{-4})$.
    \item Assuming that $\psi^*$ is further Lipschitz, in \Ccref{section_smooth_cvar} we prove that vinilla SGD suffices to achieve the $\mathcal{O}(\epsilon^{-4})$ complexity. As a special case, we propose a new divergence which is a smoothed approximation of CVaR.
    \item We conduct experiments to verify our theoretical results. We observe that our proposed methods significantly accelerate the optimization process, and also demonstrates superior test performance.
\end{itemize}

\subsection{Related work}
\textbf{Constrained DRO and Penalized DRO.} There are two existing formulations of the DRO problem: the constrained DRO and the penalized DRO. The constrained DRO formulation \ccref{constrained} has been studied in a number of works \citep{namkoong2016stochastic,shapiro2017distributionally,duchi2018learning}, while other works consider the penalty-based formulation \ccref{DRO} \citep{sinha2018certifiable,levy2020large}. From a Lagrangian perspective, the two formulations are equivalent; however, the dual objective of the constrained formulation is sometimes hard to solve as pointed out in \citep{namkoong2016stochastic,duchi2018learning}. In this paper we focus on the penalty-based version and provide the first non-asymptotic analysis in the non-convex setting. Moreover, we do not make the assumption that the loss is bounded, as assumed in \citet{levy2020large} in the convex setting.

\textbf{DRO with $\psi$-divergence.} $\psi$-divergence is one of the most common choices in DRO literature to measure the distance between probability distributions. It encompasses a variety of popular functions such as KL-divergence, $\chi^2$-divergence, and the conditional-value-at-risk (CVaR), etc. \Ccref{divtable} gives detailed descriptions for these functions.

For CVaR, \citet{namkoong2016stochastic} proposed a mirror-descent method which achieves $\mathcal{O}(\sqrt{T})$ regret. \citet{levy2020large} proposed a stochastic gradient-based method with optimal convergence rate in the convex setting. They also discussed an alternative approach based on the dual formulation which they call Dual SGM. In the non-convex setting, \citet{soma2020statistical} proposed a smoothed approximation of CVaR and obtain an $\mathcal{O}(\epsilon^{-6})$ complexity. We contribute to this line of work by proposing a different divergence with similar behavior as CVaR and an $\mathcal{O}(\epsilon^{-4})$ complexity.

For $\chi^2$ divergence, \citet{hashimoto2018fairness} considered a constrained formulation of DRO but did not provide theoretical guarantees. \citet{levy2020large} proposed algorithms based on an multi-level Monte-Carlo stochastic gradient estimator, and provide convergence guarantees in the convex setting. In contrast, we consider general smooth non-convex loss function $\ell$ and provide convergence guarantee for $\chi^2$ divergence as a special case of \Ccref{SNM_DRO}.

\textbf{Non-smooth non-convex optimization.} Conventional non-convex optimization typically focuses on smooth objective functions. For general smooth non-convex stochastic optimization, it is already known that the best possible gradient complexity for finding an $\epsilon$-approximate stationary point is $\mathcal{O}(\epsilon^{-4})$ \citep{arjevani2019lower}, which is achieved by SGD based algorithms \citep{ghadimi2013stochastic}. However, the optimization can be much harder for non-smooth non-convex objective functions, and there are limited results in this setting. \citet{ruszczynski2020stochastic} proposed a stochastic gradient-based method which converges to a stationary point with probability one, under the assumption that the feasible region is bounded. For unconstrained optimization, \citet{zhang2020complexity} showed that it is intractable to find an $\epsilon$-stationary point for some Lipschitz and Hadamard semi-differentiable function. When the function is weakly convex, \citet{davis2019stochastic} showed that the projected SGD converges to the stationary point of a Moreau envelope, and a recent work \citep{mai2020convergence} extended this result to SGD with momentum. In this paper, we show that for smooth non-convex loss $\ell$, DRO can be formulated as a non-smooth non-convex optimization problem, but the special property of the DRO objective makes it possible to find an $\epsilon$-stationary point within $\mathcal{O}(\epsilon^{-4})$ complexity.

\section{Preliminaries}
\label{preliminaries}
\subsection{Notations and Assumptions}
Throughout this paper we use $\|\cdot\|$ to denote the $\ell_2$-norm in an Euclidean space $\mathbb{R}^d$ and use $\left\langle \cdot,\cdot \right\rangle$ to denote the standard inner product.
For a real number $t$, denote $(t)_+$ as $\max(t,0)$. For a set $C$, denote $\mathbb I_C(\cdot)$ as the indicator function such that $\mathbb I_C(x)=0$ if $x\in C$ and $\mathbb I_C(x)=+\infty$ otherwise. We first list some basic definitions in optimization literature, which will be frequently used in this paper.

\begin{definition}
 (Lipschitz continuity) A mapping $f: \mathcal{X} \to \mathbb{R}^m$ is $G$-Lipschitz continuous if for any $x,y\in \mathcal X$, $\left\| f(x)-f(y) \right\| \leq G \left\| x-y\right\|$. 
\end{definition}

\begin{definition}
 (Smoothness) A function $f : \mathcal{X} \to \mathbb{R}$ is $L$-smooth if it is differentiable on $\mathcal{X}$ and the gradient $\nabla f$ is $L$-Lipschitz continuous, i.e. $\left\| \nabla f(x)-\nabla f(y) \right\| \leq L \left\| x-y\right\|$ for all $x,y\in \mathcal X$. We say $f$ is non-smooth if such $L$ does not exist.
\end{definition}

\begin{definition}
\label{def_conjugate}
 (Conjugate function) For a function $\psi: \mathbb{R} \to \mathbb{R}$, the conjugate function $\psi^*$ is defined as $\psi^*(t) := \sup_{s \in \mathbb{R}} \left( st - \psi(s) \right)$.
\end{definition}

\begin{assumption}
\label{assumption_general}
 We make the following assumptions throughout the paper:
\begin{itemize}[topsep=0pt]
\setlength{\itemsep}{0pt}
    \item Given $\xi$, the loss function $\ell(x,\xi)$ is $G$-Lipschitz continuous and $L$-smooth with respect to $x$;
    \item $\psi$ is a valid divergence function, i.e. a non-negative convex function satisfying $\psi(1)=0$ and $\psi(t)=+\infty$ for all $t<0$. Furthermore the conjugate $\psi^*$ is $M$-smooth.
\end{itemize}
\end{assumption}
We finally define the notion of $\epsilon$-stationary points for differentiable non-convex functions.
\begin{definition}
\label{def_stationary_point}
 ($\epsilon$-stationary point) For a differentiable function $f: \mathcal{X} \to \mathbb{R}$, a point $x\in\mathcal X$ is said to be first-order $\epsilon$-stationary if $\|\nabla f(x)\|\le \epsilon$.
\end{definition}

\subsection{Equivalent formulation of the DRO objective}
\label{section_equivalence}
The aim of this paper is to find an $\epsilon$-stationary point of problem \ccref{DRO}. However, the original formulation \ccref{DRO} involves a max operation over distributions which makes optimization challenging. By duality arguments we can show that the DRO objective \ccref{DRO} can be equivalently written as (see detailed derivations in ~\cite[Section A.1.2]{levy2020large})
\begin{equation}
\label{dual}
    \Psi(x) = \min_{\eta\in\mathbb{R}} \lambda\mathbb{E}_{\xi\sim P} \psi^{*}\left( \frac{\ell(x;\xi)-\eta}{\lambda} \right) + \eta.
\end{equation}
Thus, to minimize $\Psi(x)$ in \ccref{dual}, one can jointly minimize $\mathcal{L}(x,\eta) :=\mathbb{E}_{\xi\sim P} \left[ \lambda\psi^{*}\left( \frac{\ell(x;\xi)-\eta}{\lambda} \right) + \eta \right]$
over $(x,\eta)\in\mathcal{X}\times\mathbb{R}\subset\mathbb{R}^{n+1}$.
This can be seen as a standard stochastic optimization problem. The remaining thing is to show one can find an $\epsilon$-stationary point of $\Psi(x)$ by optimizing $\mathcal{L}(x,\eta)$ instead. We first present a lemma that gives connection of the gradient of $\Psi(x)$ to the gradient of $\mathcal L(x,\eta)$.
\begin{lemma}
\label{lemma:gradient}
 Under the \Ccref{assumption_general}, $\Psi(x)$ is differentiable, and $\nabla \Psi(x)=\nabla_x\mathcal L(x,\eta)$ for any $\eta\in \arg\min_{\eta'} \mathcal L(x,{\eta'})$.
\end{lemma}
Note that the $\eta$ in \Ccref{lemma:gradient} may not be unique but the values of $\nabla_x\mathcal L(x,\eta)$ are all equal. Since $\Psi(x)$ is differentiable, the $\epsilon$-stationary points are well-defined. We now prove that the problem of finding an $\epsilon$-stationary point of $\Psi(x)$ is equivalent to finding an $\epsilon$-stationary point of a rescaled version of $\mathcal L(x,\eta)$.
\begin{theorem}
\label{thm:stationary}
 Under the \Ccref{assumption_general}, if for some $(x,\eta)$ the following holds: $\|\nabla_x\mathcal{L}(x,\eta)\|+G|\nabla_{\eta}\mathcal{L}(x,\eta)| \leq\epsilon$, then $x$ is an $\epsilon$-stationary point of $\Psi(x)$. Furthermore, define a rescaled function
 \begin{equation}
 \label{L}
     \widehat{\mathcal{L}}(x,\eta)=\mathcal{L}(x,G\eta):=\mathbb{E}_{\xi\sim P} \left[ \lambda\psi^{*}\left( \frac{\ell(x;\xi)-G\eta}{\lambda} \right) + G\eta \right],
 \end{equation}
 then $\|\nabla \widehat{\mathcal{L}}(x,\eta)\|\le \epsilon/\sqrt 2$ implies that $x$ is an $\epsilon$-stationary point of $\Psi(x)$.
\end{theorem}
The proof of \Ccref{lemma:gradient} and \Ccref{thm:stationary} can be found in \Ccref{sec_generalized_gradient}. From the above theorem it suffices to find an $\epsilon$-stationary point of $\widehat{\mathcal{L}}(x,\eta)$ such that $\|\nabla \widehat{\mathcal{L}}(x,\eta)\|\le \epsilon$ (ignoring numerical constant $\sqrt 2$). As a result, we will mainly work with $\widehat{\mathcal{L}}$ in subsequent analysis.
The property of the objective function \Ccref{L} heavily depends on $\psi^*$. We list some popular choices of $\psi$ together with the corresponding $\psi^*$ in \Ccref{divtable}. They serve as motivating examples of our subsequent analysis.

\begin{table}[h]
\centering
\caption{Some commonly used divergences and the corresponding conjugates.}
\label{divtable}
\begin{tabular}{@{}ccc@{}}
\toprule
Divergence & $\psi(t)$                      & $\psi^*(t)$                 \\ \midrule
$\chi^2$   & $\frac{1}{2}(t-1)^2$           & $-1+\frac{1}{4}(t+2)_{+}^2$ \\
K-L        & $t \log t -t+1$                & $e^t-1$                     \\
CVaR       & $\mathbb{I}_{[0,\alpha^{-1})}, \alpha\in (0,1)$ & $\alpha^{-1}(t)_{+}$        \\
KL-regularized CVaR       & $\mathbb{I}_{[0,\alpha^{-1})}+t\log t-t+1, \alpha\in (0,1)$ & $\min(e^t,\alpha^{-1}(1+t+\log \alpha))-1$        \\
Cressie-Read & $\frac{t^k-kt+k-1}{k(k-1)},k \in \mathbb{R}$ & $\frac 1 k\left(\left( (k-1)t+1\right)_{+}^{\frac{k}{k-1}}-1\right)$ \\\bottomrule
\end{tabular}
\end{table}

\section{Analysis of general non-convex DRO}
\label{smooth_psi}

In this section we will analyze the DRO problem with general smooth non-convex loss functions $\ell$. We first discuss the challenges appearing in our analysis, then show how to leverage the specific structure of the objective function in order to overcome these challenges. Specifically, we show that our proposed algorithm can achieve a non-asymptotic complexity of $\mathcal{O}(\epsilon^{-4})$.

\subsection{Challenges in non-convex DRO}

A standard result in optimization literature states that if the objective function is smooth and the stochastic gradient is unbiased and has bounded variance\footnote{ $\mathbb{E}_{\xi\sim P} \|\nabla_x \ell(x,\xi) - \nabla_x \ell(x) \|^2 \le \sigma^2$ for some $\sigma$ and all $x\in\mathcal X$ where $\ell(x)=\mathbb{E}_{\xi\sim P}\ell(x,\xi)$.}, then standard stochastic gradient descent (SGD) algorithms can provably find an $\epsilon$-first-order stationary point under $\mathcal O(\epsilon^{-4})$ gradient complexity \citep{ghadimi2013stochastic}. Here the smoothness and bounded variance property are crucial for the convergence of SGD \citep{zhang2019adam}. However, we find that \emph{both} assumptions are violated in non-convex DRO, even if the \emph{inner} loss $\ell(x,\xi)$ is smooth and the stochastic noise is bounded for both $\ell(x,\cdot)$ and $\nabla_x \ell(x,\cdot)$. We present a counter example to illustrate this point, in which we can gain some insight about the structure of the DRO objective.

\begin{example}
\textup{
Consider the loss $\ell(x;\xi) = x^2 \left( 1+ \frac{\xi}{x^2+1}\right)^2$ which is a quadratic-like function with noise $\xi$, where $\xi$ is a Rademacher variable drawn from $\{-1,+1\}$ with equal probabilities. Then a straightforward calculation shows that the loss $\ell$ has the following properties:
\begin{itemize}[topsep=0pt]\setlength{\itemsep}{0pt}
    \item (Smoothness) For any $\xi\in\{-1,+1\}$,  $\ell(x,\xi)$ is $L$-smooth with respect to $x$ for $L=8$;
    \item (Bounded variance) For any $x\in \mathbb R$, $\mathbb{E}_{\xi}\left[ \left( \ell(x,\xi)-x^2 \right)^2 \right] =\frac {4x^4}{(x^2+1)^2} + \frac {x^4}{(x^2+1)^4}\leq 4$. It then follows that $\operatorname{Var}_\xi[\ell(x,\xi)]\le 4$;
    \item (Bounded variance for gradient) Similarly we can check that the gradient of $\ell$ also has bounded variance. Moreover, the variance tends to zero when $x\rightarrow\infty$.
\end{itemize}
Now consider the DRO where $\psi$ is chosen as the commonly used $\chi^2$-divergence. Fix $\lambda=1$ and $\eta=0$. Based on the expression of $\psi^*(t)$ in \Ccref{divtable}, the DRO objective function \ccref{L} thus takes the form $\widehat{\mathcal{L}}(x,0;\xi) = \frac 1 4 \left[x^2 \left( 1+ \frac{\xi}{x^2+1}\right)^2+2\right]^2-1$, which is a quartic-like function. It follows that
\begin{itemize}[topsep=0pt]\setlength{\itemsep}{0pt}
    \item $\widehat{\mathcal{L}}(x,0;\xi)=\Theta(x^4)$ for large $x$ and therefore $\widehat{\mathcal{L}}(x,0;\xi)$ is not globally smooth;
    \item $\nabla_x \widehat{\mathcal{L}}(x,0;\xi)=x^3+2x\xi+2x+\mathcal O(1)$ for large $x$ and the stochastic gradient variance $\operatorname{Var}[\nabla_x \widehat{\mathcal{L}}(x,0;\xi)]=\Theta(x^2)$ which is unbounded globally.
\end{itemize}
}
\end{example}

As we can see from the above example, both the local smoothness and the gradient variance of $\widehat{\mathcal{L}}$ strongly rely on the scale of $x$. Indeed, in general non-convex DRO both the two quantities have a positive correlation with the magnitude of $\ell$. As shown in \Ccref{sec_bounded_loss}, if we make the additional assumption that $\ell$ is bounded by a small constant, then the smoothness and gradient noise can be controlled in a straightforward way, and we show that a projected stochastic gradient method can be applied in this setting. However, such bounded loss assumption is quite restrictive and not satisfactory.

\subsection{Main results}
\label{section_snm}

In this section, we present the main theoretical result of this paper. All proofs can be founded in \Ccref{sec_proof_main}. We make the following assumption on the noise of the stochastic loss:
\begin{assumption}
\label{BV}
We assume that  for all $x \in \mathcal X$, the stochastic loss has bounded variance, i.e. $\mathbb{E}_{\xi\sim P}\left( \ell(x,\xi) - \ell(x) \right)^2 \leq \sigma^2$ where $\ell(x)=\mathbb{E}_{\xi\sim P}\ell(x,\xi)$.
\end{assumption}

We now provide formal statements of the key properties mentioned above, which show that both the gradient variance and the local smoothness can be controlled in terms of the gradient norm.

\begin{lemma}
\label{bound_var}
Under \Ccref{assumption_general,BV}, the gradient estimators of \eqref{L} satisfies the following property:
\vspace{-5pt}
\begin{equation}
    \mathbb{E}_\xi\| \nabla \widehat{\mathcal{L}}(x,\eta,\xi) - \nabla \widehat{\mathcal{L}}(x,\eta) \|^2 \leq 11 G^2M^2\lambda^{-2}\sigma^2 + 8(G^2 +  \|\nabla \widehat{\mathcal{L}}(x,\eta)\|^2 )
\end{equation}
\end{lemma}

\begin{lemma}
\label{L0-L1-smooth}
Under \Ccref{assumption_general}, for any pair of parameters $(x,\eta)$ and $(x',\eta')$, we have the following property for the gradient of $\widehat{\mathcal{L}}$:
\begin{equation}
\label{L0-L1-ineq}
    \|\nabla \widehat{\mathcal L} (x,\eta)-\nabla \widehat{\mathcal L} (x',\eta')\|\le \left(K+\tfrac L G \|\nabla \widehat{\mathcal L} (x,\eta)\|\right)\|(x-x',\eta-\eta')\|
\end{equation}
where $K= L + 2G^2\lambda^{-1}M$.
\end{lemma}
Note that \ccref{L0-L1-ineq} reduces to the standard notion of smoothness if the term $\frac L G\|\nabla \widehat{\mathcal L} (x,\eta)\|$ is absent. Thus the inequality \ccref{L0-L1-ineq} can be seen as a generalized smoothness condition. \cite{zhang2019gradient} for the first time proposed such generalized smoothness for twice-differentiable functions in a different form, and \cite{NEURIPS2020_b282d173} further gave a comprehensive analysis of algorithms for optimizing generalized smooth functions. However, all these works make strong assumptions on the gradient noise and can not be applied in our setting.

Instead, we propose to use the \textit{mini-batch normalized SGD with momentum} algorithm for non-convex DRO, shown in \Ccref{SNM}. The algorithm has been theoretically analysed in \citep{cutkosky2020momentum} for optimizing standard smooth non-convex functions. Compared with \cite{cutkosky2020momentum}, we use mini-batches in each iteration in order to ensure convergence in our setting.
\vspace{-5pt}
\begin{algorithm}[h]
\SetKwInOut{KIN}{Input}
\caption{Mini-batch Normalized SGD with Momentum}
\label{SNM}
\KIN{The objective function $F(w)$, distribution $P$, initial point $w_0$, initial momentum $m_0$, learning rate $\gamma$, momentum factor $\beta$, batch size $S$ and total number of iterations $T$ }
\For{$t \gets 1$ \textbf{to} $T$}{
    $\hat{\nabla} F(w_{t-1}) \gets \frac{1}{S} \sum_{i=1}^S \nabla F(w_{t-1},\xi_{t-1}^{(i)})$ where $\{\xi_{t-1}^{(i)}\}_{i=1}^S$ are i.i.d. samples drawn from $P$\\
	$m_t \gets \beta m_{t-1} + (1-\beta )\hat{\nabla} F(w_{t-1} )$\\
	$w_{t} \gets w_{t-1} - \gamma \dfrac{m_t}{\|m_t\|}$\\
}
\end{algorithm}
\vspace{-10pt}

The following main theorem establishes convergence guarantee of  \Ccref{SNM}. We further provide a sketch of proof in \Ccref{sketch}, where we can gain insights on how normalization and momentum techniques help tackle the difficulties shown in \Ccref{bound_var,L0-L1-smooth}. 

\begin{theorem}
\label{SNM_convergence}
Suppose that $F$ satisfies the following conditions:
\begin{itemize}[leftmargin=30pt]
    \item (Generalized smoothness) $\|\nabla F(w_1)-\nabla F(w_2)\|\le (K_0+K_1\|\nabla F(w_1)\|)\|w_1-w_2\|$ holds for any $w_1,w_2$;
    \item (Gradient variance) The stochastic gradient $\nabla F(w,\xi)$ is unbiased $(\nabla F(w) = \mathbb{E}_\xi \nabla F(w,\xi))$ and satisfies $\mathbb{E}_\xi\left\| \nabla F(w,\xi) - \nabla F(w) \right\|^2 \leq \Gamma^2 \left\| \nabla F(w) \right\|^2 + \Lambda^2$ for some $\Gamma$ and $\Lambda$.
\end{itemize}

Let $\{w_t\}$ be the sequence produced by \Ccref{SNM}. Then with a mini-batch size ${S} = {\Theta}(\Gamma^2)$ and a suitable choice of parameters $\gamma$ and $\beta$, for any small $\epsilon = \mathcal{O}(\min(K_0/K_1,\Lambda/\Gamma))$, we need at most $\mathcal{O}\left( \Delta K_0\Lambda^2\epsilon^{-4} \right)$ gradient complexity to guarantee that we find an $\epsilon$-stationary point in expectation, i.e.
$\frac{1}{T}\sum_{t=0}^{T-1} \mathbb{E}\|\nabla F(w_t)\| \leq \epsilon$ where $\Delta = F(w_0) - \inf_{w \in \mathbb{R}^d} F(w)$.
\end{theorem}

Substituting \Ccref{L0-L1-smooth,bound_var} into \Ccref{SNM_convergence} immediately yields the final result:

\begin{corollary}
\label{SNM_DRO}
Suppose the DRO problem \ccref{DRO} satisfies \Ccref{assumption_general,BV}. Using \Ccref{SNM} with a constant batch size, the gradient complexity for finding an $\epsilon$-stationary point of $\Psi(x)$ is 
\begin{equation}
    \notag
    \mathcal{O}\left( G^2\left(M^2\sigma^2\lambda^{-2}+1\right)\left( \lambda^{-1}MG^2+L \right)\Delta\epsilon^{-4} \right).
\end{equation}
\end{corollary}
\vspace{-5pt}
\Ccref{SNM_DRO} shows that \Ccref{SNM} finds an $\epsilon$-stationary point with complexity $\mathcal{O}(\epsilon^{-4})$, which is the same as standard smooth non-convex optimization. Also note that the bound in \Ccref{SNM_convergence} does not depend on $K_1$ and $\Gamma$ as long as $\epsilon$ is sufficiently small. In other words, \Ccref{SNM} is well-adapted to the non-smoothness and unbounded noise in our setting. We also point out that although the batch size is chosen propositional to $\Gamma^2$, the required number of iterations $T$ is inversely propositional to $\Gamma^2$, therefore the total number of stochastic gradient computations remains the same.

Finally, note that \Ccref{SNM_convergence} is stated in a general form and is not limited to DRO setting. It greatly extends the results in \citet{NEURIPS2020_b282d173,zhang2019gradient} by relaxing their noise assumptions, and demonstrates the effectiveness of combining adaptive gradients with momentum for optimizing ill-conditioned objective functions. More importantly, our algorithm is to some extent similar to currently widely used optimizers in practice, e.g. Adam. We believe our result can shed light on why these optimizers often show superior performance in real applications. 

\subsection{Proof sketch of \Ccref{SNM_convergence}}
\label{sketch}
Below we present our proof sketch, in which the motivation of using \Ccref{SNM} will be clear. Similar to standard analysis in non-convex optimization, we first derive a descent inequality for functions satisfying the generalized smoothness:
\begin{lemma}
\label{DesIneq}
(Descent inequality) Let $F(x)$ be a function satisfying the generalized smoothness condition in \Ccref{SNM_convergence}. Then for any point $x$ and direction $z$ the following holds:
\begin{equation}
    F\left(x-z\right) \leq F\left(x\right)-\left\langle\nabla F\left(x\right), z\right\rangle+\frac 1 2(K_{0}+K_{1}\left\|\nabla F(x)\right\|)\left\|z\right\|^{2}.
\end{equation}
\end{lemma}
\vspace{-3pt}
The above lemma suggests that the algorithm should take a small step size when $\left\|\nabla F(x) \right\|$ is large in order to decrease $F$. This is the main motivation of considering a normalized update. Indeed, after some careful calculation we can prove the following result:
\begin{lemma}
\label{SNMLemma}
Consider the algorithm that starts at $w_0$ and makes updates $w_{t+1} = w_t - \gamma \frac{m_{t+1}}{\left\|m_{t+1}\right\|}$ where $\{m_t\}$ is an arbitrary sequence of points. Define $\delta_t := m_{t+1} - \nabla F(w_t)$ be the estimation error. If $\gamma= O(1/K_1)$, then
\begin{equation}
\label{sketch_descent}
    F(w_{t})-F(w_{t+1}) \geq \left( \gamma - \frac 1 2 K_1 \gamma^2 \right) \| \nabla F(w_t) \| -  \frac 1 2 K_0\gamma^2 - 2 \gamma \| \delta_t \|
\end{equation}
\end{lemma}
\vspace{-4pt}
which is $\gamma  \| \nabla F(w_t) \|-2\gamma\|\delta_t\|-\mathcal O(\gamma^2)$ for small $\gamma$. Therefore the objective function $F(w)$ decreases if $\|\delta_t\|<1/2\cdot\| \nabla F(w_t) \|$, i.e. a small estimation error. However, $\delta_t$ is related to the stochastic gradient noise which can be very large due to \Ccref{bound_var}. This motivates us to the use the momentum technique for the choice of $\{m_t\}$ to reduce the noise. Formally, let $\beta$ be the momentum factor and define $\hat{\delta_t}=\hat{\nabla} F(w_t) -\nabla F(w_t)$, then using the recursive equation of momentum $m_t$ in \Ccref{SNM} we can show that
\begin{equation}
\label{eq_delta}
    \delta_{t}=\beta \sum_{\tau=0}^{t-1}\beta^{\tau} (\nabla F(w_{t-\tau-1})-\nabla F(w_{t-\tau}))+(1-\beta) \sum_{\tau=0}^{t-1}\beta^{\tau}  \hat{\delta}_{t-\tau}+(1-\beta)\beta^t \hat{\delta}_0.
\end{equation}
The first term of the right hand side in \ccref{eq_delta} can be bounded using the generalized smoothness condition, and the core procedure is to bound the second term using a careful analysis of conditional expectation and the independence of noises $\{\hat{\delta}_t\}$ (see \Ccref{appendix_delta_lemma} in Appendix). Finally, the use of mini-batches of size  $\Theta(\Gamma^2)$, a carefully chosen $\beta$ and a small enough $\gamma$ ensure that $\sum_{t=0}^{T-1}\|\delta_t\|< c \sum_{t=0}^{T-1} (\mathbb{E}\|\nabla F(w_t)\|+\mathcal O(\epsilon))$ where $c< 1/2$. This guarantees that the right hand side of \ccref{sketch_descent} is overall positive, and by taking summation over $t$ in \ccref{sketch_descent} we have that 
\begin{align*}
    F(w_0)-F(w_T)&\ge (1-2c)\gamma\sum_{t=0}^{T-1}\|\nabla F(w_t)\|-\mathcal O(\gamma^2T-\gamma T\epsilon).\\
    \text{namely,}\qquad \qquad&\frac 1 T\sum_{t=0}^{T-1}\|\nabla F(w_t)\|\le \mathcal O\left(\frac \Delta {\gamma T} + \gamma+\epsilon\right).
\end{align*}
Finally, for a suitable choice of $\gamma$ we can obtain the minimum gradient complexity bound on $T$.

\subsection{Dealing with the CVaR case}
\label{section_smooth_cvar}
Previous analysis applies to any divergence function $\psi$ as long as $\psi^*$ is smooth. This includes some popular choices such as the $\chi^2$-divergence, but not the CVaR. In the case of CVaR, $\psi^*$ is not differentiable as shown in  \Ccref{divtable}, which is undesirable from an optimization viewpoint. In this section we introduce a smoothed version of CVaR. The conjugate function of the smoothed CVaR is also smooth, so that the results in  \Ccref{section_snm} can be directly applied in this setting. 

For standard CVaR at level $\alpha$, $\psi_{\alpha}(t)$ takes zero when $t\in [0,1/\alpha)$ and takes infinity otherwise. Instead, we consider the following smoothed version of CVaR:
\begin{equation}
\label{cvar_smooth_phi}
    {\psi _\alpha^{\text{smo}} }(t) = \left\{ {\begin{array}{*{20}{l}}
  {t\log t + \frac{{1 - \alpha t}}{\alpha }\log \frac{{1 - \alpha t}}{{1 - \alpha }}}& t\in [0,1/\alpha) \\ 
  +\infty &\text{otherwise} 
\end{array}} \right.
\end{equation}
It is easy to see that $\psi _\alpha^{\text{smo}}$ is a valid divergence. The corresponding conjugate function is
\begin{equation}
\label{cvar_smooth_phi_dual}
    \psi^{\text{smo},*}_{\alpha}(t) = \frac 1 {\alpha} \log(1-\alpha+\alpha \exp(t)).
\end{equation}
The following propositions demonstrate that $\psi _\alpha^{\text{smo}}$ is indeed a smoothed approximation of CVaR.
\begin{proposition}
Fix $0<\alpha<1$. When $\lambda \rightarrow 0^+$, the solution of the DRO problem \ccref{L} for smoothed CVaR tends to the solution for the standard CVaR. Note that the solution of the standard CVaR does not depend on $\lambda$.
\end{proposition}

\begin{proposition}
\label{psi_cvar_smooth}
$\psi^{\text{smo},*}_{\alpha}(t)$ is $\frac 1 {\alpha}$-Lipschitz and $\frac 1 {4\alpha}$-smooth.
\end{proposition}
Based on \Ccref{psi_cvar_smooth}, we can then use \Ccref{SNM_DRO} to obtain the gradient complexity (taking $M=1/4\alpha$).

Note that $\psi^{\text{smo},*}_{\alpha}(t)$ is not only smooth but also Lipschitz. In this setting, we can in fact obtain a stronger result than the general one provided in  \Ccref{SNM_DRO}. Specifically, the gradient noise and smoothness of the objective function $\widehat{\mathcal{L}}(x,\eta,\xi)$ can be bounded, as shown in the following lemma:
\begin{lemma}
Suppose \Ccref{assumption_general} holds. For smoothed CVaR, the DRO objective \ccref{L} satisfies
\begin{equation}
    \mathbb{E}\| \nabla \widehat{\mathcal{L}}(x,\eta,\xi)\|^2 \leq 2\alpha^{-2}G^2.
\end{equation}
Moreover, $\widehat{\mathcal{L}}(x,\eta)$ is $K$-smooth with $K = \frac{L}{\alpha} + \frac{G^2}{2\lambda\alpha}$.
\end{lemma}

Equipped with the above lemma, we can obtain the following guarantee for smoothed CVaR, which shows that \emph{vanilla SGD} suffices for convergence. 

\begin{theorem}
Suppose that $\psi = \psi_{\alpha}^{\text{smo}}$ and  \Ccref{assumption_general} holds. If we run SGD with properly selected hyper-parameters on the loss $\widehat{\mathcal  L}(x,\eta)$, then the gradient complexity of finding an $\epsilon$-stationary point of $\Psi(x)$ is $\mathcal{O}\left( \alpha^{-3}\lambda^{-1} G^2(G^2+\lambda L)\Delta\epsilon^{-4} \right)$, where $\Delta = \mathcal{L}(x_0,\eta_0) - \inf_{x}\Psi(x)$.
\end{theorem}

The above theorem shows a similar convergence rate compared with \Ccref{SNM_DRO} in terms of $\epsilon$ and $G$, and the dependency on $\lambda$ is even better. Therefore the Lipschitz property of $\psi^*$ is very useful, in that it is now possible to use a simpler algorithm while achieving a similar (or even better) bound.

\section{Experiments}
\label{experiment}
We perform two sets of experiments to verify our theoretical results. In the first set of experiments, we consider the setting in \Ccref{section_snm}, where the loss $\ell(x;\xi)$ is highly non-convex and unbounded, and $\psi$ is chosen to be the commonly used $\chi^2$-divergence such that its conjugate is smooth. We will show that $\mathrm{(i)}$ the vanilla SGD algorithm cannot optimize this loss efficiently due to the non-smoothness of the DRO objective; $\mathrm{(ii)}$ by simply adopting the normalized momentum algorithm, the optimization process can be greatly accelerated. In the second set of experiments, we deal with the CVaR setting in \Ccref{section_smooth_cvar}. We will show that by employing the smooth approximation of CVaR defined in \ccref{cvar_smooth_phi} and \ccref{cvar_smooth_phi_dual}, the optimization can be greatly accelerated.

\subsection{Experimental settings}
\textbf{Tasks.} We consider two tasks: the classification task and the regression task. While classification is more common in machine learning, here we may also highlight the regression task, since recent studies show that DRO may be more suitable for non-classification problems in which the metric of interest is continuous as opposed to the 0-1 loss \citep{hu2018does,levy2020large}.

\textbf{Datasets.} We choose the AFAD-Full dataset for regression and CIFAR-10 dataset for classification. AFAD-Full \citep{niu2016ordinal} is a regression task to predict the age of human from the facial information, which contains more than 160K facial images and the corresponding age labels ranging from 15 to 75. Note that AFAD-Full is an imbalanced dataset where the ages of two thirds of the whole dataset are between 18 and 30. Following the experimental setting in \citep{chang2011ordinal,chen2013cumulative,niu2016ordinal}, we split the whole dataset into a training set comprised of 80\% data and a test set comprised of the remaining 20\% data. CIFAR-10 dataset is a classification task consisting of 10 classes with 5000 images for each class. To demonstrate the effectiveness of our method in DRO setting, we adopt the setting in \cite{chou2020remix} to construct an imbalanced CIFAR-10 by randomly sampling each category at different ratio. See Appendix for more details.
 
\textbf{Model.} For all experiments in this paper, we use the standard ResNet-18 model in \citep{he2016deep}. The output has 10 logits for CIFAR-10 classification task, and has a single logit for regression.

\textbf{Training details.} We choose the penalty coefficient $\lambda=0.1$ and the CVaR coefficient $\alpha=0.02$ in all experiments. For each algorithm, we tune the learning rate hyper-parameter from a grid search and pick the one that achieves the fastest optimization speed. The momentum factor is taken to 0.9 in all experiments, and the mini-batch size is chosen to be 128. We train the model for 100 epochs on CIFAR-10 dataset and 200 epochs on AFAD-Full dataset. Other training details can be found in \Ccref{sec_exp_details}.

\begin{figure}[]
    \centering
    \subfigure[Regression for $\chi^2$ penalized DRO]{
		\includegraphics[width=0.376\linewidth]{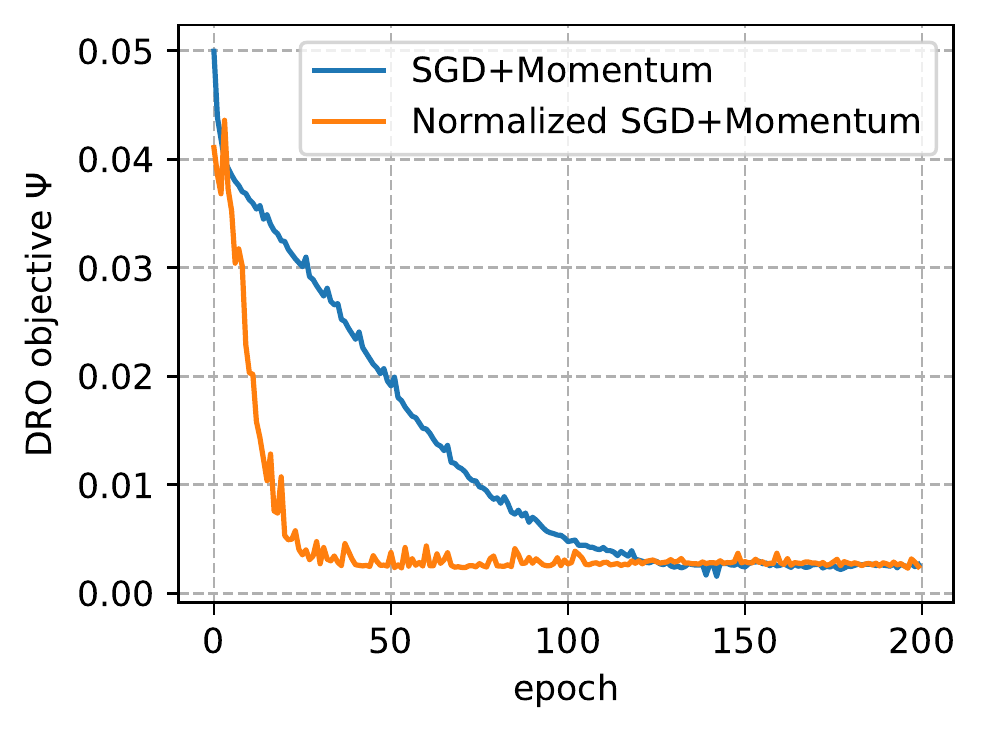}
		\label{exp1reg}
		}
	\subfigure[Classification for $\chi^2$ penalized DRO]{	
	    \includegraphics[width=0.38\linewidth]{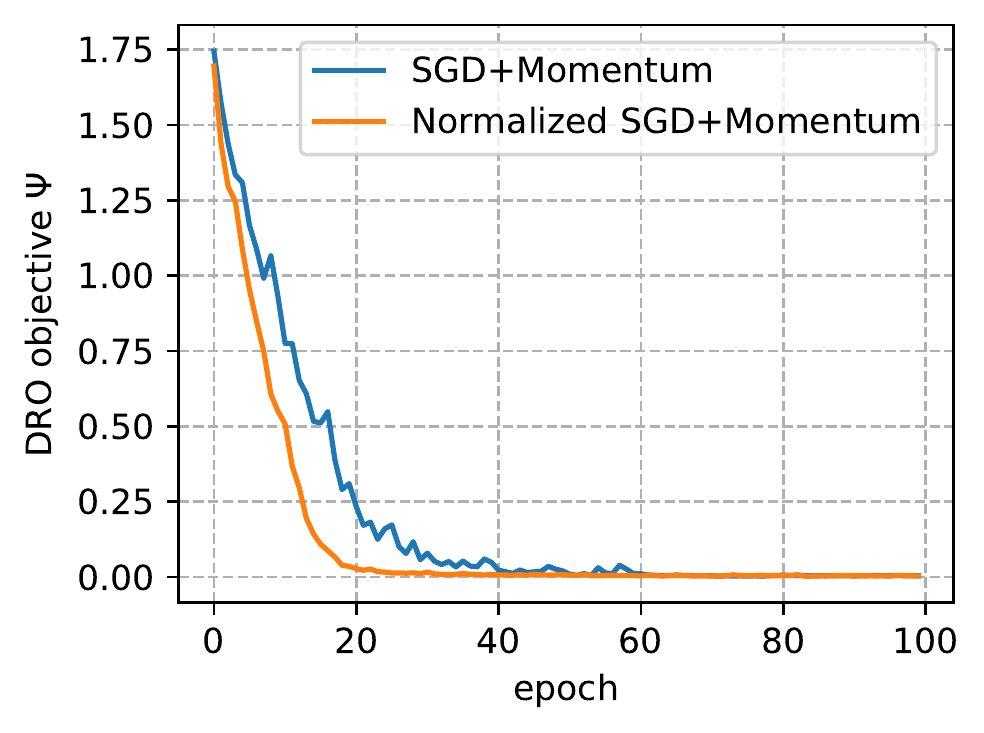}
		\label{exp1cls}
		}
	\quad
	
	\subfigure[Regression for smoothed CVaR]{
		\includegraphics[width=0.38\linewidth]{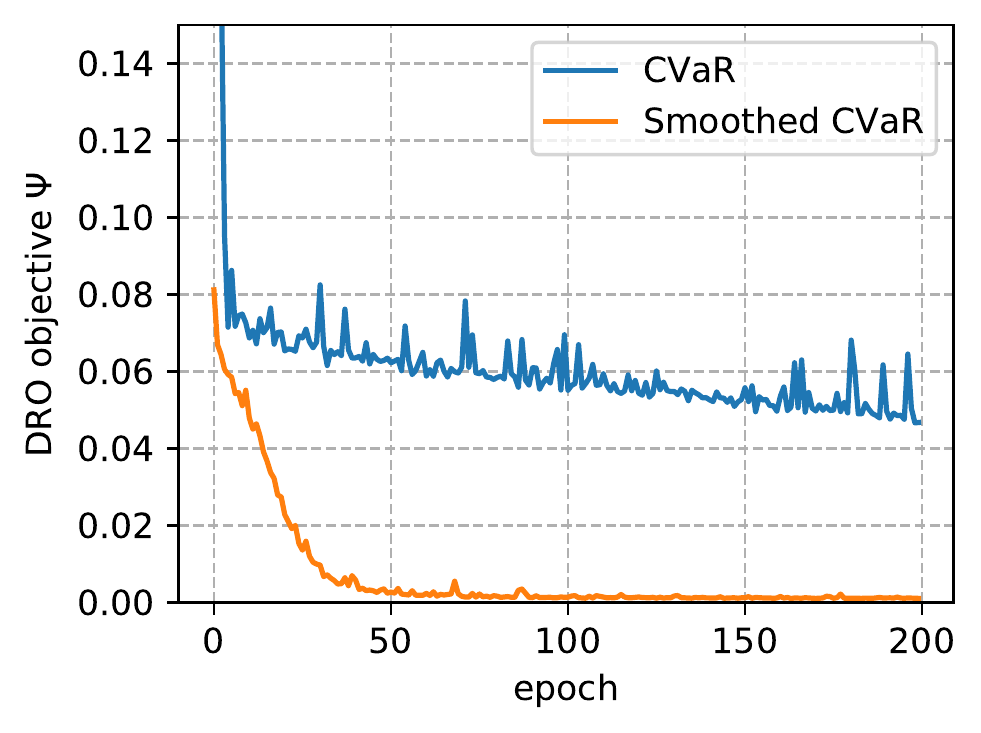}
		\label{exp2reg}
		}
	\subfigure[Classification for smoothed CVaR]{
		\includegraphics[width=0.375\linewidth]{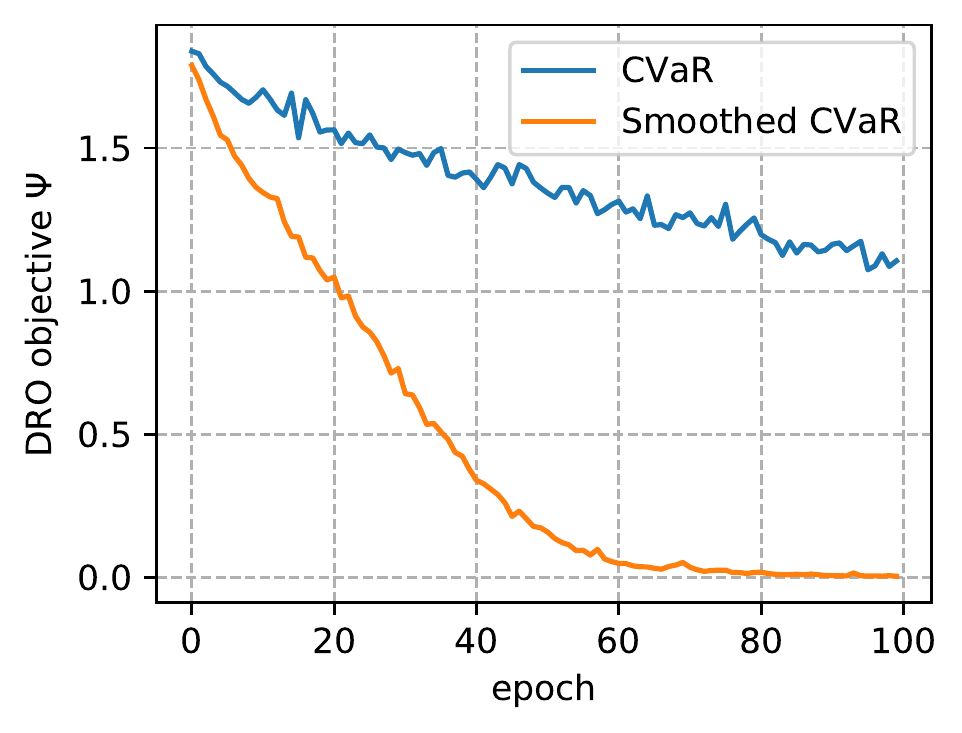}
		\label{exp2cls}
		}
		\vspace{-5pt}
	\caption{\small Training curve of $\chi^2$ penalized DRO and smoothed CVaR in regression and classification task.}
	\label{fig:exp}
	\vspace{-10pt}
\end{figure}

\subsection{Experimental results}
Results are demonstrated in \Ccref{fig:exp}. For each figure, we plot the value of the DRO objective $\Psi(x)$ through the training process. Here we calculate $\Psi(x)=\min_\eta \mathcal L(x,\eta)$ at each epoch based on a convex optimization on $\eta$ until convergence (rather than using ${\mathcal L}(x,\eta)$ with the current parameter $\eta$ directly).

\textbf{Experimental result for $\chi^2$ penalized DRO.} \Ccref{exp1reg} and \Ccref{exp1cls} plot the training curve of the DRO objective using different algorithms. It can be seen that in both regression and classification, vanilla SGD converges slowly, and using normalized momentum algorithm significantly improves the convergence speed. For example, in regression task SGD does not converge after 100 epochs while normalized momentum algorithm converges just after 25 epochs. These results highly consist with our theoretical findings, which shows that due to the non-smoothness of the DRO loss, vanilla SGD may not be able to optimize the loss well; In contrast, normalized momentum utilizes the relationship between local smoothness and gradient magnitude, and achieves better performance.

\textbf{Experimental result for smoothed CVaR.}  \Ccref{exp2reg} and \Ccref{exp2cls} plot the training curves for different training losses: CVaR and smoothed CVaR. Note that the evaluation metrics ($y$-axis) in these figures are all chosen to be CVaR, even when the training objective is smoothed CVaR. In this way we can make a fair comparison of optimization speed based on these training curves. Firstly, it can be seen that the optimization of CVaR is very hard due to the non-smoothness, and the training curves have lots of spikes. In contrast, the optimization of smoothed CVaR is much easier for both tasks, and the final loss is significantly lower. Such experimental results show the benefit of our proposed smoothed CVaR for optimization.

\textbf{Test performance}. We also measure the test performance of trained models to see whether a better optimizer can also improve test accuracy. Due to space limitation, in the main text we provide results of $\chi^2$ penalized DRO problem for classification using unbalanced CIFAR-10 dataset, which is listed in \Ccref{tab:test}. Other results can be found in \Ccref{sec_exp_details}. It can be seen that the model trained using normalized SGD with momentum achieves higher test accuracy on all class, and especially, the worst-performing class. Since the experiments in this paper is mainly designed to compare algorithms rather than to achieve best performance, better performance is likely to be reached if adjusting the hyper-parameters (e.g. $\lambda$, the number of epochs, and the learning rate schedule).
\begin{table}[h]
    \centering
    \small
    \caption{Test performance of the $\chi^2$ penalized DRO problem for unbalanced CIFAR-10 classification. Each column corresponds to the performance of a particular class. The bolded column indicates the worst-performing class.}
    \vspace{-4pt}
    \setlength\tabcolsep{4pt}
    \begin{tabular}{c|cccccccccc}\hline
        Class & 1 & 2 & 3 & 4 & 5 & \textbf{6} & 7 & 8 & 9 & 10\\ \hline
        Number of training samples & 4020 & 2715 & 4985 & 2965 & 1950 & \textbf{1425} & 4795 & 4030 & 4835 & 3300 \\
        Test acc (SGD+Momentum)& 76.7 & 80.1 & 70.2 & 55.0 & 54.6 & \textbf{44.8} & 84.9 & 77.7 & 85.5 & 76.8\\
        Test acc (Normalized SGD+Mom.)& 78.8 & 81.2 & 71.7 & 57.3 & 56.2 & \textbf{49.8} & 87.2 & 83.5 & 90.4 & 78.4\\
        \hline
    \end{tabular}
    \label{tab:test}
    \vspace{-5pt}
\end{table}

\section{Discussion}

\textbf{Conclusion.}
In this paper we provide non-asymptotic analysis of first-order algorithms for the DRO problem with unbounded and non-convex loss. Specifically, we write the original DRO problem as a non-smooth non-convex optimization problem, and we propose an efficient normalization-based algorithm to solve it. The general result of  \Ccref{SNM_convergence} might be of independent value and is not limited to DRO setting. We hope that this work can also bring inspiration to the study of other non-smooth non-convex optimization problems.

\textbf{Limitations.} 
Despite the theoretical grounds and promising experimental justifications, there are some interesting questions that remain unexplored. Firstly, it may be possible to obtain better complexities on problem-dependent parameters, e.g. $G$ and $\lambda$. Secondly, while this paper mainly considers smooth $\psi^*$, in some cases $\psi^*$ may be non-smooth (e.g. for KL-divergence) or even not continuous. In future we hope to discover approaches that can deal with more general classes of $\psi$-divergence. Finally, we are looking forward to seeing more applications of DRO in real-world problems.


\section*{Acknowledgement}
This work was supported by Key-Area Research and Development Program of Guangdong Province (No. 2019B121204008), National Key R\&D Program of China (2018YFB1402600), BJNSF (L172037) and Beijing Academy of Artificial Intelligence. Project 2020BD006 supported by PKU-Baidu Fund. Jikai Jin is partially supported by the elite undergraduate training program of School of
Mathematical Sciences in Peking University.

\bibliographystyle{plainnat}
\bibliography{reference}

\appendix

\newpage

\section{Equivalent formulation of the DRO objective}
\label{sec_generalized_gradient}
\subsection{Generalized gradient}
As can be seen, the problem \Ccref{dual} is the pointwise minima over $\eta$ for a family of smooth functions $\mathcal L(x, \eta)$. However, there exists a known result showing that the pointwise minima of a family of smooth functions may not be differentiable in general, so the gradient may not exist\footnote{For example, consider function $f(x,\eta)=\frac 1 {(\eta^2+1)}\log(1+\exp((\eta^2+1)x))$ that is jointly smooth in $(x,\eta)$. However, the pointwise minima $\min_{\eta\in\mathbb R} f(x,\eta)=\max(x,0)$ which is non-differentiable at $x=0$.}. 

We first assume $\Psi(x)$ is \emph{non-smooth} and non-convex. To measure the convergence of non-smooth non-convex optimization, we define the notion called the \emph{generalized gradient} \cite[Chapter 2]{clarke1990optimization}.
\begin{definition}
 (Local Lipschitzness) A function $f: \mathcal{X} \to \mathbb{R}^m$ is locally Lipschitz continuous near point $x\in \operatorname{int}(\mathcal X)$ if there exists $G,\epsilon>0$ such that for any $y,z \in \mathcal B_\epsilon(x)$, $| f(y)-f(z) | \leq G \left\| y-z\right\|$. Here $\mathcal B_\epsilon(x)$ denotes the set of points in the open ball of radius $\epsilon$ around $x$.
\end{definition}
\begin{definition}
(Generalized gradient)
Suppose that $f: \mathcal{X} \to \mathbb{R}$ is locally Lipschitz-continuous at $x$, where $\mathcal{X}\subset\mathbb R^n$. Its generalized directional derivative in direction $v$ is defined as
\begin{equation}
    \notag
    f^{\circ}(x; v) = \mathop{\operatorname{limsup}}_{\substack{y\to x\\ t\to 0}}\frac{f(y+tv)-f(y)}{t}
\end{equation}
and the generalized gradient at $x$ is the set
\begin{equation}
    \notag
    \partial f(x) = \left\{ \zeta\in\mathbb{R}^n : f^{\circ}(x;v) \geq \left\langle \zeta, v\right\rangle  \forall v \in\mathbb{R}^n\right\}.
\end{equation}
\end{definition}
Interested readers may refer to the book \citep{clarke1990optimization} for an in-depth exploration of this concept. Importantly, $\partial f(x)$ is a non-empty closed convex set; $\partial f(x)$ degenerates to a single point $\{\nabla f(x)\}$ if $f$ is smooth, and $\partial f(x)$ is equivalent to the sub-gradient if $f$ is convex. If $x$ is a local minima (or maxima) for $f(x)$, then $0\in\partial f(x)$. The following proposition gives the relationship between generalized gradient and (conventional) gradient.
\begin{proposition}
\label{relation_generalized_gradient}
 (\cite[Section 2.2]{clarke1990optimization}) If function $f$ is differentiable at $x$, then $f$ is local Lipschitz near $x$ and $\partial f(x)=\{\nabla f(x)\}$. Conversely, if $f$ is local Lipschitz near $x$ and $\partial f(x)$ reduces to a singleton point $\{g\}$, then $f$ is differentiable at $x$ and $\nabla f(x)=g$.
\end{proposition}

\subsection{Proof of \Ccref{lemma:gradient}}

We first present a basic lemma which provides a rule to calculate generalized gradients of the pointwise maxima of a function family.
\begin{lemma}[\citep{clarke1981generalized}]
\label{pointwise_minima}
Suppose that $\mathcal T\subset \mathbb R^m$ is compact and $\mathcal X\subset \mathbb R^n$ is open. Let $f: \mathcal X \times \mathcal T \to \mathbb{R}$ be a $K$-Lipschitz continuous function in $x\in \mathcal X$ for some $K$ and is continuous in $t\in \mathcal T$. Define the point-wise minima function $F(x) = \min_{t\in \mathcal T} f(x,t)$, then we have
\begin{equation}
    \label{point_maxima}
    \partial F(x) \subset \operatorname{Conv}\bigcup_{t\in T(x)}\partial_{x}f(x,t)
\end{equation}
where $T(x)=\left\{ t\in T: F(x)=f(x,t)\right\}$, $\partial_{x}f(x,t)$ is the partial generalized gradient and $\operatorname{Conv}$ denotes a convex hull of a point set.
\end{lemma}

Recall that in our setting$\Psi(x) = \min_{\eta\in\mathbb{R}}\mathcal{L}(x,\eta)$. Since $\psi^*$ and $\ell$ are differentiable, $\mathcal{L}$ is differentiable in both $\eta$ and $x$. To make use of \Ccref{pointwise_minima}, we have to constrain $\eta$ in a compact set $\mathcal T$. This is possible if we constrain $x$ in a compact set $\mathcal B_r(x_0)$, an open ball of radius $r$ centered at $x_0$.

\begin{lemma}
 Assume \Ccref{assumption_general} holds. Fix a point $x_0\in\mathcal X$. Denote $\eta_0\in\operatorname{argmin}_\eta \mathcal L(x_0,\eta)$ be an arbitrary minima. Then for any point $x\in\mathcal B_r(x_0)$ near $x_0,$ there exists $\eta_x\in \operatorname{argmin}_\eta \mathcal L(x,\eta)$, such that $|\eta_0-\eta_x|\le Gr$.
\end{lemma}
\begin{proof}
Using the condition that $\psi^*$ is convex and differentiable, we have $\eta_x\in \operatorname{argmin}_\eta \mathcal L(x,\eta)$ if and only if $\nabla_\eta \mathcal L(x,\eta)=0$. Namely,
\begin{equation}
    \nabla_\eta \mathcal L(x,\eta)=1-\mathbb E_\xi\left[(\psi^*)'\left(\frac{\ell(x;\xi)-\eta_x}{\lambda}\right)\right]=0 \quad \text{iff}\quad \eta_x\in \operatorname{argmin}_\eta \mathcal L(x,\eta).
\end{equation}
For any point $x\in\mathcal B_r(x_0)$, $|\ell(x;\xi)-\ell(x_0;\xi)|\le Gr$ holds for any $\xi$ due to the Lipschitz property of $\ell(\cdot;\xi)$. Considering that $(\psi^*)'$ is monotonically increasing (due to the convexity of $\psi^*$), we have
\begin{align}
    1-\mathbb E_\xi\left[(\psi^*)'\left(\frac{\ell(x;\xi)-(\eta_0+Gr)}{\lambda}\right)\right]&\ge 0\\
    1-\mathbb E_\xi\left[(\psi^*)'\left(\frac{\ell(x;\xi)-(\eta_0-Gr)}{\lambda}\right)\right]&\le 0
\end{align}
Since $\nabla_\eta \mathcal L(x,\eta)$ is continuous in $\eta$, there must exists an $\eta_x\in[\eta_0-Gr,\eta_0+Gr]$, such that 
\begin{equation}
    1-\mathbb E_\xi\left[(\psi^*)'\left(\frac{\ell(x;\xi)-\eta_x}{\lambda}\right)\right]= 0
\end{equation}
Therefore $\eta_x\in\operatorname{argmin}_\eta \mathcal L(x,\eta)$.
\end{proof}

For any point $x_0$, we can use \Ccref{pointwise_minima} by substituting $\mathcal X=\mathcal B_r(x_0)$ and $\mathcal T=[\eta_0-Gr,\eta_0+Gr]$. The procedure is as follows:
\begin{itemize}[topsep=0pt]
\setlength{\itemsep}{0pt}
    \item $\Psi(x)=\min_{\eta\in\mathbb R}\mathcal L(x,\eta)=\min_{\eta\in [\eta_0-Gr,\eta_0+Gr]}\mathcal L(x,\eta)$ holds for all $x\in \mathcal B_r(x_0)$;
    \item Applying \Ccref{pointwise_minima} we obtain
    \begin{equation*}
    \begin{aligned}
    \partial \Psi(x)&\subset \operatorname{Conv}\{\nabla_x \mathcal L(x,\eta):\eta\in [\eta_0-Gr,\eta_0+Gr]\cap \operatorname{argmin}_\eta \mathcal L(x,\eta)\}\\
    &\subset \operatorname{Conv}\{\nabla_x \mathcal L(x,\eta):\eta\in \operatorname{argmin}_\eta \mathcal L(x,\eta)\}
    \end{aligned}
    \end{equation*}
\end{itemize}
We finally prove below (\Ccref{lamma_singleton}) that $\{\nabla_x \mathcal L(x,\eta):\eta\in \operatorname{argmin}_\eta \mathcal L(x,\eta)\}$ is a singleton set. Then \Ccref{relation_generalized_gradient} indicates that $\Psi(x)$ is differentiable, and the generalized gradient reduces to gradient such that $\nabla \Psi(x)=\nabla_x \mathcal L(x,\eta)$ for any $\eta\in \operatorname{argmin}_\eta \mathcal L(x,\eta)$. Thus we complete the proof of \Ccref{lemma:gradient}.

\begin{lemma}
 \label{lamma_singleton}
 Assume \Ccref{assumption_general} holds. For any $\eta_1,\eta_2 \in \operatorname{argmin}_\eta \mathcal L(x,\eta)$, we have $\nabla_x \mathcal L(x,\eta_1)=\nabla_x \mathcal L(x,\eta_2)$.
\end{lemma}
\begin{proof}
Denote $X(x,\eta)$, $Y(x)$ be two random functions defined by $$X(x,\eta)=(\psi^*)'\left(\frac {\ell(x;\xi)-\eta}{\lambda}\right)\quad Y(x)=\nabla_x \ell(x;\xi)$$ which depend on the random variable $\xi$. Rewrite the gradient of $\mathcal L(x,\eta)$ as follows:
\begin{align}
    \nabla_x \mathcal L(x,\eta)&=\mathbb E[X(x,\eta)Y(x)]\\
    \nabla_\eta \mathcal L(x,\eta)&=1-\mathbb E[X(x,\eta)]
\end{align}
Note that $(\psi^*)'$ is monotonically increasing (due to the convexity of $\psi^*$), thus $X(x,\eta)$ is monotonically decreasing in $\eta$. It follows that
$$\nabla_\eta \mathcal L(x,\eta_1)=\nabla_\eta \mathcal L(x,\eta_2)\quad \text{iff}\quad \mathbb E[X(x,\eta_1)]=\mathbb E[X(x,\eta_2)] \quad \text{iff}\quad X(x,\eta_1)=X(x,\eta_2) \quad \text{a.s.}$$
Therefore $\mathbb E[X(x,\eta_1)Y(x)]=\mathbb E[X(x,\eta_2)Y(x)]$, namely $\nabla_x \mathcal L(x,\eta_1)=\nabla_x \mathcal L(x,\eta_2)$.
\end{proof}

\subsection{Proof of \Ccref{thm:stationary}}

Now, suppose that we have obtained a pair $(x,\eta)$ s.t. $\left\|\nabla_{x}\mathcal{L}(x,\eta)\right\|+G \left|\nabla_{\eta}\mathcal{L}(x,\eta)\right| \leq \epsilon$. Let $x$ be fixed and $\eta^* \in \mathop{\arg\min}_{\eta} \mathcal{L}(x,\eta)$. Then we have
\begin{equation*}
    \begin{aligned}
    &\quad \left\| \nabla_x\mathcal{L}(x,\eta)-\nabla_x\mathcal{L}(x,\eta^*)\right\| \\
    &=\left\|\mathbb E_\xi\left[\left(\left(\psi^*\right)'\left(\frac{\ell(x;\xi)-\eta}{\lambda}\right)-\left(\psi^*\right)'\left(\frac{\ell(x;\xi)-\eta^*}{\lambda}\right)\right)\nabla \ell(x;\xi)\right]\right\|\\
    &\leq G \cdot\mathbb{E}_{\xi} \left| \left(\psi^*\right)'\left(\frac{\ell(x;\xi)-\eta}{\lambda}\right)-\left(\psi^*\right)'\left(\frac{\ell(x;\xi)-\eta^*}{\lambda}\right)\right| \\
    &= G\cdot\left| \mathbb{E}_{\xi} \left[ \left(\psi^*\right)'\left(\frac{\ell(x;\xi)-\eta}{\lambda}\right) - \left(\psi^*\right)'\left(\frac{\ell(x;\xi)-\eta^*}{\lambda}\right) \right]\right| \\
    &= G \left| \nabla_{\eta}\mathcal{L}(x,\eta)-\nabla_{\eta}\mathcal{L}(x,\eta^*)\right|
    = G|\nabla_{\eta}\mathcal{L}(x,\eta)|
    \end{aligned}
\end{equation*}
where we use the fact that $(\psi^*)'$ is monotone increasing (due to the comvexity of $\psi^*$. Hence, using \Ccref{lemma:gradient} we obtain
\begin{equation}
    \notag
    \left\|\nabla\Psi(x)\right\| = \left\|\nabla_x\mathcal{L}(x,\eta^*)\right\| \leq \left\|\nabla_x\mathcal{L}(x,\eta)\right\| + G \left| \nabla_\eta\mathcal{L}(x,\eta)\right| \leq \epsilon
\end{equation}
Now suppose that $\|\nabla \widehat{\mathcal L}(x,\eta)\|\le \epsilon/\sqrt 2$. Then
\begin{equation*}
    \|\nabla \widehat{\mathcal L}(x,\eta)\|^2=\|\nabla_x {\mathcal L}(x,G\eta)\|^2+G^2|\nabla_\eta {\mathcal L}(x,G\eta)|^2\le \epsilon^2/2
\end{equation*}
Using $(a+b)^2\le 2(a^2+b^2)$ we obtain
\begin{equation*}
    (\|\nabla_x {\mathcal L}(x,G\eta)\|+G|\nabla_\eta {\mathcal L}(x,G\eta)|)^2\le \epsilon^2
\end{equation*}
which completes the proof.

\section{The Stochastic Projected Gradient Descent algorithm for DRO with bounded loss}
\label{sec_bounded_loss}
In this section we use a simple projected gradient method to minimize the DRO objective \eqref{L} and analyze its convergence rate under the assumption that the loss is bounded. Since this section is not so related to the main result in our paper, we mainly provide the gradient complexity bound in terms of $\epsilon$ for finding an $\epsilon$-stationary point without delving into problem-dependent parameters.

\begin{assumption}
\label{bounded_loss}
We have $0 \leq \ell(x,\xi) \leq B$ for all $x \in \mathcal{X}$ and $\xi$.
\end{assumption}

It turns out that we can restrict the feasible region to $\mathcal{X}\times \left[{U},{V} \right]$ where $ \left[{U},{V} \right]$ is a finite interval.

\begin{proposition}
\label{finite_interval}
Under the  \Ccref{assumption_general,bounded_loss}, the DRO problem is equivalent to
\begin{equation}
\label{problem}
    \operatorname{minimize } \widehat{\mathcal{L}}(x,\eta) \qquad \text{on} \quad (x,\eta) \in \mathcal{X}\times [U,V]
\end{equation}
where $U=-\frac {\lambda C_{\psi}} G$ and $V=\frac {B-\lambda C_{\psi}} G$  are real numbers and $C_{\psi}$ is a constant depending only on $\psi$.
\end{proposition}

\begin{proof}

Note that $(\psi^*)'$ is a function satisfying the following properties:
\begin{itemize}
    \item $(\psi^*)'$ is monotonically increasing;
    \item $0\le\lim\limits_{s\to -\infty}(\psi^*)'(s)\le 1$. This is because $\lim\limits_{s\to -\infty}\frac {\psi^*(s)} s=\lim\limits_{s\to -\infty}\inf_{t\ge 0} t-\frac {\psi(t)} s=\min\{t:\psi(t)<+\infty\}\in [0, 1]$ since $\psi(1)=0$;
    \item $\lim\limits_{s\to +\infty}(\psi^*)'(s)\ge 1$ (possibly be $+\infty$). This is because $\frac {\psi^*(s)} s=\sup_{t\ge 0} t-\frac {\psi(t)} s\ge 1$ for $s>0$ since $\psi(1)=0$.
\end{itemize}
Therefore there exists a constant $C_{\psi}$ depending only on $\psi$ such that $(\psi^*)'(C_\psi)=1$. 

For any $x\in \mathcal X$, the optimal $\eta^*$ satisfies the following equation:
\begin{equation}
     \mathbb E\left[( \psi^* )' \left( \frac{\ell(x;\xi)-G\eta^*}{\lambda} \right)\right] = 1.
\end{equation}
We now show there exists an optimal $\eta^*$ such that $G\eta^*\in [-\lambda C_{\psi}, B-\lambda C_{\psi}]$. In fact, we have
\begin{itemize}
    \item For any $G\eta<-\lambda C_{\psi}$, $\mathbb E\left[( \psi^* )' \left( \frac{\ell(x;\xi)-G\eta}{\lambda} \right)\right]\ge \mathbb E\left[( \psi^* )' \left( \frac{-\eta}{\lambda} \right)\right]\ge ( \psi^* )' (C_{\psi})=1$;
    \item For any $G\eta>B-\lambda C_{\psi}$, $\mathbb E\left[( \psi^* )' \left( \frac{\ell(x;\xi)-G\eta}{\lambda} \right)\right]\le \mathbb E\left[( \psi^* )' \left( \frac{B-\eta}{\lambda} \right)\right]\le ( \psi^* )' (C_{\psi})=1$.
\end{itemize}
We conclude the proof by noting that $\mathbb E\left[( \psi^* )' \left( \frac{\ell(x;\xi)-G\eta}{\lambda} \right)\right]$ is monotonically decreasing in $\eta$.
\end{proof}

Since $\eta$ is constrained in a finite interval $[U,V]$, we propose to solve \ccref{problem} using the randomized stochastic projected gradient (RSPG) algorithm \citep{ghadimi2016mini}. It is summarized in  \Ccref{RSPG}. Note that the algorithm can deal with situations when the feasible set $\mathcal X\in \mathbb R^d$ is also constrained. 

\begin{algorithm}[!htbp]
\SetKwInOut{KIN}{Input}
\caption{Randomized stochastic projected gradient (RSPG)}\label{RSPG}
\KIN{Feasible region $\mathcal{K}$, objective function $F(w)$, distribution $P$, initial point $w_0\in\mathcal{K}$, step size $\gamma$, mini-batch sizes $S$, and total number of iterations $T$}
\For{$t \gets 1$ {to} $T$}{
    $\{\xi_{t-1}^{(i)}\}_{i=1}^S \gets \text{i.i.d. samples drawn from } P$\;
    $\hat{\nabla} F(w_{t-1}) \gets \frac{1}{S} \sum_{i=1}^S \nabla F(w_{t-1},\xi_{t-1}^{(i)})$\;
    $w_{t} \gets \Pi_{\mathcal{K}}(w_{t-1} - \gamma \hat{\nabla} F(w_{t-1}) ) \text{ where }\Pi_{\mathcal{K}} \text{is the projection onto }\mathcal{K}$\;
}
\SetKwInOut{KOUT}{Output}
\KOUT{randomly return one $w_t$ in $\{w_t\}_{t=1}^T$}
\end{algorithm}

\begin{proposition}
\label{smooth}
Suppose  \Ccref{assumption_general} holds. Under  \Ccref{finite_interval}, $\mathcal{L}$ is $K$ smooth on $\mathcal X \times [U,V]$, where $K$ only depends on $\psi,\lambda,M,B,G$ and $L$.
\end{proposition}
\begin{proof}
First note that $\left(\psi^*\right){'}$ is $M$-Lipschitz continuous, and the range of $\frac{\ell(x,\xi)-G\eta}{\lambda}$ lies in the interval $\left[ C_{\psi}-\lambda^{-1}B, C_{\psi}+\lambda^{-1}B\right]$, thus $\left( \psi^* \right){'} \left( \frac{\ell(x;\xi)-G\eta}{\lambda} \right)$ is bounded by a constant $\left|\left( \psi^* \right){'} \left( \frac{\ell(x;\xi)-G\eta}{\lambda} \right)\right|\le \left( \psi^* \right){'}( C_{\psi}+\lambda^{-1}B) $.
\begin{equation}
\notag
    \begin{aligned}
    &\quad \|\nabla_x \mathcal{L}(x_1,\eta_1) - \nabla_x \mathcal{L}(x_2,\eta_2) \| \\
    &= \left\|\mathbb{E}_{\xi\sim P} \left[ \left( \psi^* \right){'} \left( \frac{\ell(x_1;\xi)-G\eta_1}{\lambda} \right) \cdot \nabla \ell(x_1,\xi) - \left( \psi^* \right){'} \left( \frac{\ell(x_2;\xi)-G\eta_2}{\lambda} \right) \cdot \nabla \ell(x_2,\xi)\right]\right\| \\
    &\leq \mathbb{E}_{\xi\sim P} \left[  \left( \psi^* \right){'}( C_{\psi}+\lambda^{-1}B)\left\| \nabla\ell(x_1,\xi) - \nabla\ell(x_2,\xi) \right\| \right]\\
    &\quad +\mathbb E_{\xi\sim P}\left[ G \left| \left( \psi^* \right){'} \left( \frac{\ell(x_1;\xi)-G\eta_1}{\lambda} \right) - \left( \psi^* \right){'} \left( \frac{\ell(x_2;\xi)-G\eta_2}{\lambda} \right) \right| \right] \\
    &\leq  \left( \psi^* \right){'}( C_{\psi}+\lambda^{-1}B) L \|x_1-x_2\| + \lambda^{-1}GM \left( G\|x_1-x_2\|+G\left| \eta_1 - \eta_2 \right| \right)
    \end{aligned}
\end{equation}
Similarly we can show that 
\begin{equation}
     \|\nabla_{\eta} \mathcal{L}(x_1,\eta_1) - \nabla_{\eta} \mathcal{L}(x_2,\eta_2) \|\le  G\lambda^{-1}M \left( G\|x_1-x_2\|+G\left| \eta_1 - \eta_2 \right| \right)
\end{equation}
Therefore $\mathcal{L}$ is smooth.
\end{proof}

\begin{proposition}
\label{unbiased}
Suppose  \Ccref{assumption_general} holds. Under  \Ccref{finite_interval}, the stochastic gradients are unbiased estimates of the true gradients $\nabla_x \mathcal{L}$ and $\nabla_{\eta} \mathcal{L}$ and are uniformly bounded over $\mathcal{X}\times {[U,V]}$, by a constant $\Lambda$ which only depends on $\psi,\lambda,M,B,G$ and $L$.
\end{proposition}
\begin{proof}
As we have shown in the proof of  \Ccref{smooth}, the term $\left( \psi^* \right){'} \left( \frac{\ell(x;\xi)-G\eta}{\lambda} \right)$ is bounded by $\left( \psi^* \right){'}( C_{\psi}+\lambda^{-1}B) $. Then it's easy to see that $\nabla_{x}\mathcal{L}(x,\eta;\xi)$ and $\nabla_{\eta}\mathcal{L}(x,\eta;\xi)$ are bounded and the squared norm of true gradient is bounded by $\Lambda^2=2\left[( \psi^* ){'}( C_{\psi}+\lambda^{-1}B)\right]^2 G^2+G^2$.
\end{proof}

Following \citep{ghadimi2016mini,reddi2016stochastic}, in constrained optimization we typically consider a generalized gradient defined as
\begin{equation}
\notag
    \mathcal{P}_{\mathcal{X}}(x,\nabla f(x),\gamma) = \frac{1}{\gamma} (x-x^+),\quad \text{where } x^+ = \arg\min_{u \in \mathcal{X}} \left\{ \left\langle \nabla f(x), u \right\rangle + \frac{1}{2\gamma}\|u-x\|^2 \right\}
\end{equation}
Note that $x^+$ is exactly the projection of $x-\gamma \nabla f(x)$ onto the set $\mathcal{X}$. For unconstrained optimization when $\mathcal{X} = \mathbb{R}^d$, this definition coincides with the gradient in the traditional sense. We say that $x$ is an $\epsilon$-stationary point if $\left\|\mathcal{P}_{\mathcal{X}}(x,\nabla f(x),\gamma)\right\| \leq \epsilon$. The above propositions combined with ~\cite[Corollary 3]{ghadimi2016mini} imply the following convergence result.

\begin{theorem}
\label{complexity}
Suppose  \Ccref{assumption_general,bounded_loss} hold. With $\mathcal K=\mathcal X\times [U,V]$, $w_0=(x_0,\eta_0)$ and properly chosen $\gamma$ and $S$, \Ccref{RSPG} finds an $\epsilon$-stationary point with complexity $\mathcal{O}(\Lambda^2K\Delta\epsilon^{-4})$,where $\Delta = \mathcal{L}(x_0,\eta_0)-\inf_{(x,\eta)\in \mathcal{X}\times \mathbb{R}} \mathcal{L}(x,\eta)$ and $ K,\Lambda$ are constants that appeared in  \Ccref{smooth,unbiased}. Moreover, with the choice $T = 4K\Delta\epsilon^{-2}$, $\gamma=1/2L$ and $S = 24\Lambda^2\epsilon^{-2}$,  \Ccref{RSPG} finds an $\epsilon$-stationary point with probability $\geq 0.5$.
\end{theorem}

\begin{proof}
~\cite[Corollary 3]{ghadimi2016mini}, combined with  \Ccref{unbiased} implies that if $\gamma=1/2L$,
\begin{equation}
\label{grad_bound}
    \mathbb{E}\left[ \left\| \mathcal{P}_{\mathcal X \times {[U,V]}}((x_k,\eta_k),\nabla\mathcal{L}(x_k,\eta_k),\gamma)  \right\|^2 \right] \leq \frac{K\Delta}{T} + \frac{6\Lambda^2}{S}.
\end{equation}
For any $\epsilon >0$, we choose $T = 2K\Delta\epsilon^{-2}$ and $S = 12\Lambda^2\epsilon^{-2}$, then \ccref{grad_bound} implies that
\begin{equation}
    \mathbb{E}\left[ \left\| \mathcal{P}_{\mathcal X \times {[U,V]}}((x_k,\eta_k),\nabla\mathcal{L}(x_k,\eta_k),\gamma)  \right\|^2 \right] \leq \epsilon
\end{equation}
Thus the sample complexity of Algorithm 1 for finding $\epsilon$-stationary point is upper bounded by $24K\Lambda^2\Delta\epsilon^{-4}$. In this case, with pobability $\geq 0.5$ the gradient norm is upper bounded by $2\epsilon$, the conclusion follows.
\end{proof}

While the above theorem provides non-asymptotic convergence rate to a stationary point, note that the definition of generalized gradient involves the interval $[{U,V}]$ which was constructed artificially for  \Ccref{RSPG}, thus $\mathbb{E}\left[ \| \mathcal{P}_{\mathcal X \times {[U,V]}}((x_k,\eta_k),\nabla\widehat{\mathcal{L}}(x_k,\eta_k),\gamma)  \|^2 \right] \leq \epsilon$ does not necessarily lead to an $\epsilon$-stationary point of $\nabla\widehat{\mathcal{L}}$. We then show below that the generalized gradient is indeed equal to the true gradient in the unconstrained case $\mathcal{X} = \mathbb{R}^n$, therefore \Ccref{complexity} corresponds to the gradient complexity for finding an $\epsilon$-stationary point of $\Psi(x)$. 

\begin{theorem}
\label{grad_norm}
Consider the unconstrained case $\mathcal{X}=\mathbb{R}^n$.
Choose $$\tilde U=-\frac {\lambda C_\psi}{G}-\frac \epsilon L, \quad\tilde V=\frac {B-\lambda C_\psi}{G}+\frac \epsilon L$$ as the interval constraint for $\eta$.
Using parameters specified in \Ccref{complexity},  \Ccref{RSPG} arrives at $(x,\eta)$ with $\left\|\nabla\Psi(x)\right\| \leq \epsilon$ with probability $\geq 0.5$.
\end{theorem}

\begin{proof}

It suffices to show that: whenever $\|\mathcal{P}_{\mathbb{R}^n{\times [\tilde U,\tilde V]}}((x,\eta),\nabla \widehat{\mathcal{L}}(x,\eta),\gamma) \| \leq \epsilon$, we must have $\| \nabla \widehat{\mathcal{L}}(x,\eta) \| \leq \epsilon$.

Recall that
\begin{equation}
    \mathcal{P}_{\mathbb{R}^n{\times [\tilde U,\tilde V]}}((x,\eta),\nabla \widehat{\mathcal{L}}(x,\eta),\gamma) = \frac{1}{\gamma}(x-x^+,\eta-\eta^+)
\end{equation}
where 
\begin{equation}
    \begin{aligned}
        x^+ &= \arg\min_{u \in\mathbb{R}^n} \left\{ \left\langle \nabla_x \widehat{\mathcal{L}}(x,\eta),u \right\rangle + \frac{1}{2\gamma}\|u-x\|^2 \right\} \\
        \eta^+ &= \arg\min_{\rho\in{[\tilde U,\tilde V]}} \left\{ \rho\nabla_{\eta} \widehat{\mathcal{L}}(x,\eta) + \frac{1}{2\gamma}(\rho-\eta)^2 \right\}
    \end{aligned}
\end{equation}
Define $ \eta_0 := \eta - \gamma \nabla_{\eta}\widehat{\mathcal{L}}(x,\eta)$. Since $\|\mathcal{P}_{\mathbb{R}^n{\times [\tilde U,\tilde V]}}((x,\eta),\nabla \widehat{\mathcal{L}}(x,\eta),\gamma) \| \leq \epsilon$, we have $\left| \eta - \eta^+ \right| \leq \gamma\epsilon$. We consider two possible cases below:
\begin{itemize}
    \item \textbf{Case 1. } $\eta^+ \in {(\tilde U,\tilde V)}$. In this case it is easy to see that $\eta^+=\eta_0$ and thus
    \begin{equation}
        \notag
        \| \nabla \widehat{\mathcal{L}}(x,\eta) \| = \|\mathcal{P}_{\mathbb{R}^n{\times [\tilde U,\tilde V]}}((x,\eta),\nabla \widehat{\mathcal{L}}(x,\eta),\gamma) \| \leq \epsilon
    \end{equation}
    \item \textbf{Case 2. } $\eta^+ \in \{ {\tilde U,\tilde V} \}$. Assume that $\eta^+ = \tilde{U}$ (the case $\eta^+ = \tilde{V}$ is similar). Then $\eta \in [\tilde{U},\tilde{U}+\gamma\epsilon]$. Note that $\tilde{U}+\gamma\epsilon= -\frac {\lambda C_\psi}{G}+\frac {\epsilon}{2L}<U$. In this case, $\left( \psi^* \right)' \left( \frac{\ell(x,\xi)-\eta}{\lambda} \right) \ge 1$. Therefore
    \begin{equation}
    \notag
        \eta_0 = \eta - \gamma \nabla_{\eta}\widehat{\mathcal{L}}(x,\eta) = \eta - G\gamma \left( 1-\mathbb{E}\left[\left( \psi^* \right)'\left( \frac{\ell(x,\xi)-G\eta}{\lambda} \right) \right] \right)\ge \eta.
    \end{equation}
    However, $\eta^+\le \eta$, therefore it can only be that $\eta^+=\eta=\eta_0$. Therefore we still have 
    \begin{equation}
        \notag
        \| \nabla \widehat{\mathcal{L}}(x,\eta) \| = \|\mathcal{P}_{\mathbb{R}^n{\times [\tilde U,\tilde V]}}((x,\eta),\nabla \widehat{\mathcal{L}}(x,\eta),\gamma) \| \leq \epsilon
    \end{equation}
\end{itemize}
\end{proof}

In the above theorem, the constraint of $\eta$ is $[\tilde U,\tilde V]$ which strictly contains $\eta\in [U,V]$ (in \Ccref{finite_interval}). Nevertheless, the difference of the endpoints between $U$($V$) and $\tilde U$($\tilde V$) is only $\mathcal O(\epsilon)$. Therefore it does not change the final gradient complexity of $\mathcal O(\epsilon^4)$ in \Ccref{complexity}.
\newpage

\section{Proofs in \Ccref{section_snm}}
\label{sec_proof_main}
In this section we present the proof of main results in Section \ref{section_snm}. For convenience we restate the results before proving them.

\subsection{Proofs of \Ccref{bound_var,L0-L1-smooth}}
\begin{lemma}
Under \Ccref{assumption_general,BV}, the gradient estimators of \eqref{L} satisfies the following property:
\vspace{-5pt}
\begin{equation}
    \mathbb{E}_\xi\| \nabla \widehat{\mathcal{L}}(x,\eta,\xi) - \nabla \widehat{\mathcal{L}}(x,\eta) \|^2 \leq 11 G^2M^2\lambda^{-2}\sigma^2 + 8(G^2 +  \|\nabla \widehat{\mathcal{L}}(x,\eta)\|^2 )
\end{equation}
\end{lemma}

\begin{proof}
For a random vector $X$, define the sum of its element-wise variance as
\begin{equation}
    \mathbb{V}\left(X \right) := \mathbb{E}\left\| X - \mathbb{E}[X] \right\|_2^2,
\end{equation}
Then it is easy to check that, for i.i.d. random vectors $X_1,X_2$ we have 
\begin{equation}
    \label{cal_var}
    \mathbb{E}\left\| X_1 - X_2 \right\|^2 = 2 \mathbb{V}[X_1].
\end{equation}
We first bound the variance of the stochastic gradient $\nabla_x \widehat{\mathcal{L}}(x,\eta;\xi)$. Indeed we have
\begin{align}
    &\quad \mathbb{V}\left[ \nabla_x \widehat{\mathcal{L}}(x,\eta;\xi) \right] \\
    &= \frac 1 2\mathbb{E}_{\xi_1,\xi_2} \left\| \left( \psi^* \right){'} \left( \frac{\ell(x;\xi_1)-G\eta}{\lambda} \right) \cdot \nabla \ell(x,\xi_1) - \left( \psi^* \right){'} \left( \frac{\ell(x;\xi_2)-G\eta}{\lambda} \right) \cdot \nabla \ell(x,\xi_2) \right\|^2 \\
    \notag
    &\le \mathbb{E}_{\xi_1,\xi_2} \left[ \left( \left( \psi^* \right){'} \left( \frac{\ell(x;\xi_1)-G\eta}{\lambda} \right) \right)^2 \left\| \nabla \ell(x,\xi_1) - \nabla \ell(x,\xi_2) \right\|^2 \right] \\
    \label{proof_c1}
    &\quad + \mathbb{E}_{\xi_1,\xi_2} \left[ \left\| \nabla\ell(x,\xi_2) \right\|^2 \left( \left( \psi^* \right){'} \left( \frac{\ell(x;\xi_1)-G\eta}{\lambda} \right) - \left( \psi^* \right){'} \left( \frac{\ell(x;\xi_2)-G\eta}{\lambda} \right) \right)^2 \right] \\
    \label{proof_cx}
    &\leq 4G^2\mathbb{E}_{\xi_1} \left[ \left( \left( \psi^* \right){'} \left( \frac{\ell(x;\xi_1)-G\eta}{\lambda} \right) \right)^2 \right] +  G^2M^2\lambda^{-2} \mathbb{E}_{\xi_1,\xi_2} \left[ \left( \ell(x,\xi_1) - \ell(x,\xi_2) \right)^2 \right] \\
    \label{nabla_x}
    &\leq 4G^2 \mathbb{E}_{\xi_1} \left[ \left( \left( \psi^* \right){'} \left( \frac{\ell(x;\xi_1)-G\eta}{\lambda} \right) \right)^2 \right] + 2G^2M^2\lambda^{-2}\sigma^2
\end{align}
Here in \ccref{proof_c1} we use that fact that $(a+b)^2\le 2(a^2+b^2)$ for any $a,b$; in \ccref{proof_cx} we use \Ccref{assumption_general}. Now we deal with the first term. Using $2(a-1)^2+2 \geq a^2$ for any $a$, we have
\begin{equation}
\label{ft}
    \begin{aligned}
    \mathbb{E}_{\xi} \left[ \left( \left( \psi^* \right){'} \left( \frac{\ell(x;\xi)-G\eta}{\lambda} \right) \right)^2 \right] 
    &\leq 2 + 2 \mathbb{E}_{\xi} \left[ \left( 1 - \left(  \psi^* \right){'} \left( \frac{\ell(x;\xi)-G\eta}{\lambda} \right) \right)^2 \right] \\
    &\leq 2 \left( 1 +  G^{-2}\| \nabla_{\eta} \widehat{\mathcal{L}}(x,\eta) \|^2 + G^{-2}\mathbb{V}[ \nabla_{\eta} \widehat{\mathcal{L}}(x,\eta;\xi) ] \right)
    \end{aligned}
\end{equation}
Next, $\mathbb{V}[ \nabla_{\eta} \widehat{\mathcal{L}}(x,\eta;\xi) ]$ can be easily bounded as follows:
\begin{equation}
\label{proof_c2}
    \begin{aligned}
    \mathbb{V}[ \nabla_{\eta} \widehat{\mathcal{L}}(x,\eta;\xi) ]
    &= \frac{1}{2} G^2\mathbb{E}_{\xi_1,\xi_2} \left[ \left( \left( \psi^* \right){'} \left( \frac{\ell(x;\xi_1)-G\eta}{\lambda} \right) - \left( \psi^* \right){'} \left( \frac{\ell(x;\xi_2)-G\eta}{\lambda} \right) \right)^2 \right] \\
    &\leq G^2M^2\lambda^{-2}\sigma^2
    \end{aligned}
\end{equation}
Combining with \ccref{nabla_x,ft,proof_c2}, we obtain 
\begin{equation}
\notag
    \begin{aligned}
    \mathbb{V}[ \nabla_x \widehat{\mathcal{L}}(x,\eta;\xi) ] 
    &\leq 2G^2M^2\lambda^{-2}\sigma^2 + 8 ( G^2 + \|\nabla_{\eta} \widehat{\mathcal{L}}(x,\eta)\|^2 + G^2M^2\lambda^{-2}\sigma^2 ) \\
    &= 10G^2M^2\lambda^{-2}\sigma^2+ 8( G^2+\|\nabla_{\eta} \widehat{\mathcal{L}}(x,\eta)\|^2) \\
    &\leq 10G^2M^2\lambda^{-2}\sigma^2+ 8( G^2+ \|\nabla \widehat{\mathcal{L}}(x,\eta)\|^2)
    \end{aligned}
\end{equation}
Finally,
\begin{equation}
\notag
    \begin{aligned}
    \mathbb{V}[ \nabla \widehat{\mathcal{L}}(x,\eta;\xi) ] 
    &= \mathbb{V}[ \nabla_x \widehat{\mathcal{L}}(x,\eta;\xi) ] + \mathbb{V}[ \nabla_{\eta}\widehat{\mathcal{L}}(x,\eta;\xi) ] \\
    &\leq 11 G^2M^2\lambda^{-2}\sigma^2 + 8( G^2+ \|\nabla \widehat{\mathcal{L}}(x,\eta)\|^2)
    \end{aligned}
\end{equation}
\end{proof}

\begin{lemma}
Under \Ccref{assumption_general}, for any pair of parameters $(x,\eta)$ and $(x',\eta')$, we have the following property for the gradient of $\widehat{\mathcal{L}}$:
\begin{equation}
    \|\nabla \widehat{\mathcal L} (x,\eta)-\nabla \widehat{\mathcal L} (x',\eta')\|\le \left(K+\tfrac L G \|\nabla \widehat{\mathcal L} (x,\eta)\|\right)\|(x-x',\eta-\eta')\|
\end{equation}
where $K= L + 2G^2\lambda^{-1}M$.
\end{lemma}
\begin{proof}
First write $\nabla \widehat{\mathcal L}(x,\eta)$ as
\begin{equation}
    \nabla \widehat{\mathcal L}(x,\eta) = \mathbb E_{\xi} \left[\left((\psi^*)'\left(\frac {\ell(x;\xi)-G\eta} {\lambda}\right)\nabla \ell(x,\xi),G-G(\psi^*)'\left(\frac {\ell(x;\xi)-G\eta} {\lambda}\right)\right)^T\right]
\end{equation}
We then split $\nabla \mathcal L(x,\eta)-\nabla \mathcal L(x',\eta')$ into two terms $A+B$, where
\begin{equation}
\label{decouple_smoothness}
\begin{aligned}
    A&=\mathbb E_{\xi} \left[ \left((\psi^*)'\left(\frac {\ell(x;\xi)-G\eta} {\lambda}\right)(\nabla \ell(x;\xi)-\nabla\ell(x';\xi)),0\right)^T\right]\\
    B&=\mathbb E_{\xi} \left[ \left((\psi^*)'\left(\frac {\ell(x;\xi)-G\eta} {\lambda}\right)-(\psi^*)'\left(\frac {\ell(x';\xi)-G\eta'} {\lambda}\right)\right)\left(\nabla \ell(x';\xi),-G\right)^T\right].
\end{aligned}
\end{equation}
$A$ can be bounded as follows:
\begin{equation}
\label{smooth_A}
    \|A\|\le L\cdot \mathbb E_{\xi}\left[(\psi^*)'\left(\frac {\ell(x;\xi)-G\eta} {\lambda}\right)\|x-x'\|\right]
\end{equation}
where we use $(\psi^*)'(s)\ge 0$ for all $s$. $B$ can be bounded as follows:
\begin{equation}
\label{smooth_B}
\begin{aligned}
    \|B\|&\le \sqrt 2 G \cdot\mathbb E_{\xi}\left[\left|(\psi^*)'\left(\frac {\ell(x;\xi)-G\eta} {\lambda}\right)-(\psi^*)'\left(\frac {\ell(x';\xi)-G\eta'} {\lambda}\right)\right|\right]\\
    &\le \sqrt 2 G\lambda^{-1} M \mathbb E_{\xi}\left[\left|(\ell(x;\xi)-\ell(x';\xi))-G(\eta-\eta')\right|\right]\\
    &\le 2G^2\lambda^{-1} M \|(x,\eta)^T-(x',\eta')^T\|
\end{aligned}
\end{equation}
where the last step is because the function $(x,\eta)\to \ell(x,\xi)-G\eta$ is $\sqrt 2 G$ Lipschitz. Finally we bound $\mathbb E_{\xi}\left[(\psi^*)'\left(\frac {\ell(x;\xi)-G\eta} {\lambda}\right)\right]$ using the true gradient of $\widehat{\mathcal{L}}$: 
\begin{equation*}
    \mathbb E_{\xi}\left[(\psi^*)'\left(\frac {\ell(x;\xi)-G\eta} {\lambda}\right)\right]= 1-G^{-1}\nabla_{\eta}\widehat{\mathcal L}(x,\eta)\le 1+G^{-1}|\nabla_{\eta}\widehat{\mathcal L}(x,\eta)|
\end{equation*}
Combining the above inequalities, we obtain
\begin{equation*}
\begin{aligned}
    &\quad\|\nabla \widehat{\mathcal L} (x,\eta)-\nabla \widehat{\mathcal L} (x',\eta')\|
    \le \|A\|+\|B\|\\
    &\le(L+LG^{-1})|\nabla_\eta\widehat{\mathcal L}(x,\eta)|\|x-x'\|+ 2G^2\lambda^{-1} M \|(x-x',\eta-\eta')^T\|\\
    &\le\left(L+2G^2\lambda^{-1} M+\tfrac L G \|\nabla \widehat{\mathcal L} (x,\eta)\|\right)\|(x-x',\eta-\eta')\|
\end{aligned}
\end{equation*}
which concludes the proof.
\end{proof}

\subsection{Proof of Theorem 3.5}

\subsubsection{Properties of generalized smoothness}
We formalize the generalized smoothness property into a definition.
\begin{definition}
A continuously differentiable function $F: \mathbb{R}^d \to \mathbb{R}$ is said to be $(K_0,K_1)$-smooth if $\left\| \nabla F(x) - \nabla F(y) \right\| \leq (K_0+K_1 \left\| \nabla F(x) \right\|)\|x-y\|$ for all $x,y \in \mathbb{R}^d$.
\end{definition}

We now present a descent  inequality for $(K_0,K_1)$-smooth functions which will be used in subsequent analysis.

\begin{lemma}
\label{appendix_dec_ineq}
(Descent Inequality) Let $F$ be $(K_0,K_1)$-smooth, then for any point $x$ and direction $z$ the following holds:
\begin{equation}
    F\left(x-z\right) \leq F\left(x\right)-\left\langle\nabla F\left(x\right), z\right\rangle+\frac 1 2(K_{0}+K_{1}\left\|\nabla F(x)\right\|)\left\|z\right\|^{2}.
\end{equation}
\end{lemma}
\begin{proof}
By definition we have
\begin{equation}\begin{aligned}
 F(x-z) - F(x) - \left\langle z, \nabla F(x) \right\rangle 
&= \int_{0}^{1} \left\langle \nabla F \left(x-\theta z \right) - \nabla F(x), z \right\rangle \text{d}\theta \\
&\le \int_{0}^{1} \| \nabla F \left(x-\theta z \right) - \nabla F(x)\|\| z \| \text{d}\theta\\
&\leq \int_{0}^{1} \left( K_0\theta\|z\|^{2} + K_1\theta \|z\|^2 \|\nabla F(x)\| \right) \text{d}\theta \\
&= \dfrac{K_0+K_1\|\nabla F(x)\|}{2}\|z\|^{2}
\end{aligned}\end{equation}
so the conclusion follows.
\end{proof}

\subsubsection{Properties of the normalized update}
We begin with a simple algebraic lemma.

\begin{lemma}
\label{lemma_mom_clip_1}
Let $\mu\ge 0$ be a real constant. For any vectors $u$ and $v$,
\begin{align}
     \label{lemma_mom_clip_1_1}
    -\frac {\left\langle u, v\right\rangle} {\|v\|}\le -\mu\|u\|-(1-\mu)\| v\|+(1+\mu)\|v-u\|
\end{align}
\end{lemma}
\begin{proof}
\begin{equation*}
\begin{aligned}
    -\frac {\left\langle u, v\right\rangle} {\|v\|}
    &=- \|v\| + \frac {\left\langle v-u, v\right\rangle} {\|v\|}\\
    &\le-\|v\|+\|v-u\|\\
    &\le -\|v\| + \|v-u\|+\mu(\|v-u\|+\|v\|-\|u\|)\\
    &= -\mu \|u\| -(1-\mu) \|v\| + (1+\mu) \|v-u\|
\end{aligned}
\end{equation*}
\end{proof}

Now we can characterize the behavior of normalization-based algorithms in terms of function value descent.

\begin{lemma}
\label{appendix_SNMLemma}
Consider the algorithm that starts at $w_0$ and makes updates $w_{t+1} = w_t - \gamma \frac{m_{t+1}}{\left\|m_{t+1}\right\|}$ where $\{m_t\}$ is an arbitrary sequence of points. Define $\delta_t := m_{t+1} - \nabla F(w_t)$ be the estimation error. Then
$$ F(w_{t+1})-F(w_t) \leq-\left( \gamma - \frac{1}{2}K_1 \gamma^2 \right) \| \nabla F(w_t) \| +\frac 1 2 K_0\gamma^2 + 2 \gamma \| \delta_t \| $$
And thus by a telescope sum we have
$$\left( 1-\frac{1}{2}K_1 \gamma \right) \sum_{t=0}^{T-1} \| \nabla F(w_t) \| \leq \dfrac{F(w_0)-F(w_T)}{\gamma} + \frac 1 2 K_0T \gamma + 2 \sum_{t=0}^{T-1} \| \delta_t \| $$
\end{lemma}
\begin{proof}
Since $\|w_{t+1}-w_t\| = \gamma$, by  \Ccref{appendix_dec_ineq} we have
\begin{equation}\begin{aligned}
F(w_{t+1})-F(w_t) &\leq - \frac{\gamma}{\| m_{t+1} \|} \left\langle \nabla F(w_t), m_{t+1} \right\rangle + \frac{1}{2}\gamma^2 \left(K_0+K_1 \|\nabla F(w_t)\| \right) \\
&\leq \gamma \left( -\|\nabla F(w_t)\|+2\|\delta_t\| \right) + \dfrac{1}{2}\gamma^2 \left(K_0+K_1 \|\nabla F(w_t)\| \right)  \\
&= -\left( \gamma - \frac{1}{2}K_1 \gamma^2 \right) \| \nabla F(w_t) \| +\frac 1 2 K_0\gamma^2 + 2 \gamma \| \delta_t \| \notag
\end{aligned}\end{equation}
where in the second inequality we use \Ccref{lemma_mom_clip_1}.
\end{proof}

\subsubsection{A general convergence result}
Instead of directly focusing on the specific problem of DRO, we first provide convergence guarantee for \Ccref{RSPG} under general smoothness and noise assumptions.

\begin{theorem}
Suppose that $F$ is $(K_0,K_1)$-smooth and  the stochastic gradient estimator $\nabla F(w,\xi)$ is unbiased and satisfies
\begin{equation}
    \notag
    \mathbb{E}\left\| \nabla F(w,\xi) - \nabla F(w) \right\|^2 \leq \Gamma^2 \left\| \nabla F(w) \right\|^2 + \Lambda^2
\end{equation}
Let $\{w_t\}$ be the sequence produced by \Ccref{SNM}, then with a mini-batch size ${S} = 64\Gamma^2$ and a suitable choice of parameters $\gamma$ and $\beta$, for any small $\epsilon\le \mathcal \min\left(\frac {K_0}{K_1},\frac{\Lambda}{2\Gamma}\right)$, we need at most $512 \Delta K_0\Lambda^2\epsilon^{-4}$ gradient complexity to guarantee that we find an $2\epsilon$-first-order stationary point in expectation, i.e.
$\frac{1}{T}\sum_{t=0}^{T-1} \mathbb{E}\|\nabla F(w_t)\| \leq 2\epsilon$ where $\Delta = F(w_0) - \inf_{w \in \mathbb{R}^d} F(w)$.
\end{theorem}

\begin{proof}
Define the estimation errors $\delta_t := m_{t+1} - \nabla F(w_t)$. Denote $H(a,b) := \nabla F(a) - \nabla F(b)$. We can upper bound $H(a,b)$ using the definition of $(K_0,K_1)$-smoothness:
\begin{equation}
\label{Sab}
\|H(a,b)\| \leq \|a-b\| \left( K_0+K_1 \|\nabla F(a) \| \right)
\end{equation}
Using the definition of momentum $m_t$ and $H(a,b)$, we can get a recursive formula on $\delta_t$:
\begin{equation}
\label{proof_c04}
\begin{aligned}
    \delta_{t+1}&=\beta m_{t+1}+(1-\beta)\hat{\nabla} F(w_{t+1})-\nabla F(w_{t+1})\\
    &= \beta \delta_t +\beta H\left(w_{t}, w_{t+1}\right)+(1-\beta)( \hat{\nabla} F(w_{t+1})-\nabla F(w_{t+1}))
\end{aligned}
\end{equation}
Denote $\hat{\delta}_t=\hat{\nabla} F(w_t) -\nabla F(w_t)$ be the stochastic noise, then the variance of $\hat{\delta}_t$ can be bounded by $\mathbb E\|\hat\delta_t\|^2\le\frac{1}{S}\left( \Gamma^2 \|\nabla F(w_t)\|^2+\Lambda^2 \right)$. After applying \ccref{proof_c04} recursively and plugging $\hat{\delta}_t$ into \ccref{proof_c04} we obtain
$$\delta_{t}=\beta \sum_{\tau=0}^{t-1}\beta^{\tau} H\left(w_{t-\tau-1}, w_{t-\tau}\right)+(1-\beta) \sum_{\tau=0}^{t-1}\beta^{\tau}  \hat{\delta}_{t-\tau}+(1-\beta)\beta^t \hat{\delta}_0+\beta^{t+1}(m_0-\nabla F(w_0))$$
Using triangle inequality and plugging in the estimate \ccref{Sab}, we have
\begin{equation}
\label{snm_x}
    \left\|\delta_{t}\right\| \leq (1-\beta)\left\|\sum_{\tau=0}^{t}\beta^{\tau} \hat{\delta}_{t-\tau}\right\|+\beta \gamma \sum_{\tau=0}^{t-1}\beta^{\tau}\left(K_0+K_1 \|\nabla F(w_{t-\tau-1}) \| \right)+\beta^{t+1}\|m_0-\nabla F(w_0)\|
\end{equation}
Taking a telescope summation of \ccref{snm_x} we obtain
\begin{equation}
\label{proof_c08}
    \sum_{t=0}^{T-1} \| \delta_t \|\le (1-\beta)\sum_{t=0}^{T-1}\left\|\sum_{\tau=0}^{t}\beta^{\tau} \hat{\delta}_{t-\tau}\right\|+\frac {K_0T\gamma\beta}{1-\beta} +\frac{K_1\gamma\beta}{1-\beta} \sum_{t=0}^{T-1} \left\|\nabla F\left(w_{t}\right)\right\|+\frac {\beta}{1-\beta} \|m_0-\nabla F(w_0)\|
\end{equation}
Now we take expectation of $\left\|\sum_{\tau=0}^{t}\beta^{\tau} \hat{\delta}_{t-\tau}\right\|$ over all the randomness. We will prove a core lemma ( \Ccref{appendix_delta_lemma}) later which shows

\begin{equation}
\label{proof_c07}
    \mathbb E\left\|\sum_{\tau=0}^{t}\beta^{\tau} \hat{\delta}_{t-\tau}\right\|
    \le \frac {\Lambda} {\sqrt {(1-\beta^2)S}} + \frac {\Gamma} {\sqrt S} \sum_{\tau=0}^t \beta^{\tau}\mathbb E[\|\nabla F(w_{t-\tau})\|]
\end{equation}

Now substituting \ccref{proof_c07} into \Ccref{proof_c08} we obtain
\begin{equation}
\label{proof_c09}
\begin{aligned}
    \mathbb{E} \left[ \sum_{t=0}^{T-1} \| \delta_t \| \right]
    \le& \frac {K_0T\gamma\beta}{1-\beta} +\frac{K_1\gamma\beta}{1-\beta} \sum_{t=0}^{T-1}\mathbb E \left\|\nabla F\left(w_{t}\right)\right\|+\frac {\beta}{1-\beta} \|m_0-\nabla F(w_0)\|\\
    &+ \frac {\Lambda T\sqrt {1-\beta}}{\sqrt {S}}+ \frac {\Gamma}{\sqrt S}\sum_{t=0}^{T-1} \mathbb E\left\| \nabla F(w_{t}) \right\|
\end{aligned}
\end{equation}

Finally we substitute \ccref{proof_c09} into \Ccref{appendix_SNMLemma}:
\begin{align*}
    &\quad \left( 1-\left(\dfrac{1}{2}+\frac {2\beta} {1-\beta}\right) K_1 \gamma - \frac{2\Gamma}{\sqrt{S}} \right) \mathbb E\sum_{t=0}^{T-1} \| \nabla F(w_t) \| \\
    &\leq \frac{\Delta}{\gamma} + \frac 1 2 K_0T \gamma + 2\left(  \frac{\sqrt{1-\beta}T\Lambda}{\sqrt{S}} +\dfrac{K_0T\gamma\beta}{1-\beta}+\frac {\beta}{1-\beta} \|m_0-\nabla F(w_0)\| \right)
\end{align*}

If we choose $\gamma = \frac 1 {8}(\min(K_1^{-1}, K_0^{-1}\epsilon)(1-\beta)$, and $S = 64\Gamma^2$, then
$$\left( 1-\left(\frac{1}{2}+\frac {2\beta} {1-\beta}\right) K_1 \gamma \right) - \frac{2\Gamma}{\sqrt{S}} =\left( 1-\frac {1+3\beta}{2(1-\beta)} K_1 \gamma \right)-\frac 1 4\ge \frac 3 4-\frac {2K_1\gamma}{1-\beta}\ge \frac 1 2$$
In this case
\begin{equation*}
\begin{aligned}
    \frac 1 T \mathbb E\sum_{t=0}^{T-1} \| \nabla F(w_t) \| 
    &\le 2 \left(\frac{\Delta}{\gamma T} + \frac 1 2 K_0 \gamma +\frac{2K_0\gamma\beta }{1-\beta}+\frac{\sqrt{1-\beta}\Lambda} {4\Gamma}  +\frac {2\beta}{(1-\beta)T}\|m_0-\nabla F(w_0)\| \right)\\
    &\le 2 \left(\frac{\Delta}{\gamma T} + \frac 1 {4} \epsilon +\frac{\sqrt{1-\beta}\Lambda} {4\Gamma}  +\frac {2\beta}{(1-\beta)T}\|m_0-\nabla F(w_0)\| \right)
\end{aligned}
\end{equation*}
Set $1-\beta = \min(4 \Lambda^{-2}\Gamma^2\epsilon^{2},1)$ and $m_0=\|\nabla F(w_0)\|$, then
$$\frac 1 T \mathbb E\sum_{t=0}^{T-1} \| \nabla F(w_t) \| \le \frac 3 2\epsilon + \frac{2\Delta}{\gamma T}$$

Therefore for $T=\frac{4\Delta}{\gamma \epsilon}$, we have $\frac 1 T \mathbb E\sum_{t=0}^{T-1} \| \nabla F(w_t) \| \le 2\epsilon$. The total gradient complexity is 
$$ST=\frac {2048\Gamma^2\Delta \max(K_1,K_0\epsilon^{-1})} {\min(4\Gamma^{2}\Lambda^{-2}\epsilon^2,1)\epsilon}.$$
If $\epsilon\le \mathcal \min\left(\frac {K_0}{K_1},\frac{\Lambda}{2\Gamma}\right)$, then the gradient complexity is $512\Lambda^2\Delta K_0\epsilon^{-4}$.
\end{proof}

\begin{corollary}
Suppose the DRO problem \ccref{DRO} satisfies \Ccref{assumption_general,BV}. Using \Ccref{SNM} with a constant batch size 4096, the gradient complexity for finding an $\epsilon$-stationary point of $\Psi(x)$ is 
\begin{equation}
    \notag
    \mathcal{O}\left( G^2\left(M^2\sigma^2\lambda^{-2}+1\right)\left( \lambda^{-1}MG^2+L \right)\Delta\epsilon^{-4} \right).
\end{equation}
\end{corollary}
\begin{proof}
 \Ccref{bound_var,L0-L1-smooth} imply that the conditions in \Ccref{SNM_convergence} for $\widehat{\mathcal L}(x,\eta)$ are satisfied with $K_0 = L + 2G^2\lambda^{-1}M, \Gamma^2 = 64, \Lambda^2=11 G^2M^2\lambda^{-2}\sigma^2 + 8G^2$. The main result immediately follows from \Ccref{SNM_convergence,thm:stationary}.
\end{proof}

We now return to prove the core lemma that is used in \Ccref{proof_c07}.
\begin{lemma}
\label{appendix_delta_lemma}
Let $\hat{\delta}_t=\hat{\nabla} F(w_t) -\nabla F(w_t)$ be the stochastic noise. Then
\begin{equation}
    \mathbb E\left\|\sum_{\tau=0}^{t}\beta^{\tau} \hat{\delta}_{t-\tau}\right\|
    \le \frac {\Lambda} {\sqrt {(1-\beta^2)S}} + \frac {\Gamma} {\sqrt S} \sum_{\tau=0}^t \beta^{\tau}\mathbb E[\|\nabla F(w_{t-\tau})\|].
\end{equation}
\end{lemma}
\begin{proof}
We prove the following result: for each $i\in \{0,1,\cdots,t+1\}$, the following inequality holds:
\begin{equation}
\label{proof_c30}
    \mathbb E\left\|\sum_{\tau=0}^{t}\beta^{\tau} \hat{\delta}_{t-\tau}\right\|
    \le \frac{\Gamma}{\sqrt{S}} \sum_{\tau=t-i+1}^t \beta^{t-\tau}\mathbb E\|\nabla F(w_{\tau})\|+\mathbb E\left[\sqrt{\frac {\Lambda^2}{S}\sum_{\tau=t-i+1}^t \beta^{2(t-\tau)}+\left\|\sum_{\tau=i}^t \beta^{\tau}\hat{\delta}_{t-\tau}\right\|^2}\right].
\end{equation}
It is easy to see that \Ccref{appendix_delta_lemma} follows by setting $i=t+1$ in \ccref{proof_c30}.

We prove \ccref{proof_c30} by induction. When $i=0$, \ccref{proof_c30} holds obviously. Now suppose \ccref{proof_c30} holds for $i$, and we want to prove that \ccref{proof_c30} holds for $i+1$.

Let $\mathcal{F}_t$ denote the $\sigma$-algebra generated by the stochastic gradient noise in the first $t$ iterations, i.e. $\{\xi_{\tau}^{(i)}:i\in \{1,\cdots,S\},\tau\in \{0,\cdots,t\}\}$ in  \Ccref{SNM}. We use $\mathbb{E}_{t}$ to denote the conditional expectation on $\mathcal{F}_{t}$. In other words, $\mathbb{E}_{t}$ takes expectation over the randomness in subsequent $T-t$ iterations after the first $t$ iterations finish and become deterministic. We also use $\mathbb{E}_{\mathcal F_t}$ to denote the expectation on $\mathcal F_t$. We have
\begin{align}
    &\quad \mathbb E\left[\sqrt{\frac {\Lambda^2}{S}\sum_{\tau=t-i+1}^t \beta^{2(t-\tau)}+\left\|\sum_{\tau=i}^t \beta^{\tau}\hat{\delta}_{t-\tau}\right\|^2}\right]\\
    \label{proof_c31}
    &=\mathbb E_{\mathcal F_{t-i-1}}\left[\mathbb E_{{t-i-1}}\left[\sqrt{\frac {\Lambda^2}{S}\sum_{\tau=t-i+1}^t \beta^{2(t-\tau)}+\left\|\sum_{\tau=i}^t \beta^{\tau}\hat{\delta}_{t-\tau}\right\|^2}\right]\right]\\
    \label{proof_c32}
    &\le \mathbb E_{\mathcal F_{t-i-1}}\left[\sqrt{\mathbb E_{{t-i-1}}\left[\frac {\Lambda^2}{S}\sum_{\tau=t-i+1}^t \beta^{2(t-\tau)}+\left\|\sum_{\tau=i}^t \beta^{\tau}\hat{\delta}_{t-\tau}\right\|^2\right]}\right]\\
    \label{proof_c33}
    &\le \mathbb E_{\mathcal F_{t-i-1}}\left[\sqrt{\mathbb E_{{t-i-1}}\left[\frac {\Lambda^2}{S}\sum_{\tau=t-i+1}^t \beta^{2(t-\tau)}+\beta^{2i}\|\hat{\delta}_{t-i}\|^2+\left\|\sum_{\tau=i+1}^t \beta^{\tau}\hat{\delta}_{t-\tau}\right\|^2\right]}\right]\\
    \label{proof_c34}
    &\le\mathbb E_{\mathcal F_{t-i-1}}\left[\sqrt{\mathbb E_{{t-i-1}}\left[\frac {\Lambda^2}{S}\sum_{\tau=t-i+1}^t \beta^{2(t-\tau)}+\frac {\beta^{2i}}{S}(\Gamma^2\|\nabla F(w_{t-i})\|^2+\Lambda^2)+\left\|\sum_{\tau=i+1}^t \beta^{\tau}\hat{\delta}_{t-\tau}\right\|^2\right]}\right]\\
    \label{proof_c35}
    &=\mathbb E_{\mathcal F_{t-i-1}}\left[\sqrt{\frac {\beta^{2i}}{S}\Gamma^2\|\nabla F(w_{t-i})\|^2+\frac {\Lambda^2}{S}\sum_{\tau=t-i}^t \beta^{2(t-\tau)}+\left\|\sum_{\tau=i+1}^t \beta^{\tau}\hat{\delta}_{t-\tau}\right\|^2}\right]\\
    \label{proof_c36}
    &\le\mathbb E_{\mathcal F_{t-i-1}}\left[\frac {\beta^{i}}{\sqrt{S}}\Gamma\|\nabla F(w_{t-i})\|+\sqrt{\frac {\Lambda^2}{S}\sum_{\tau=t-i}^t \beta^{2(t-\tau)}+\left\|\sum_{\tau=i+1}^t \beta^{\tau}\hat{\delta}_{t-\tau}\right\|^2}\right]\\
    \label{proof_c37}
    &= \frac {\beta^{i}}{\sqrt{S}}\Gamma \mathbb E \left[\|\nabla F(w_{t-i})\|\right]+\mathbb E\left[\sqrt{\frac {\Lambda^2}{S}\sum_{\tau=t-i}^t \beta^{2(t-\tau)}+\left\|\sum_{\tau=i+1}^t \beta^{\tau}\hat{\delta}_{t-\tau}\right\|^2}\right]
\end{align}
Here in \ccref{proof_c31} we use the property of conditional expectation; In \ccref{proof_c32} we use $\mathbb E[X^2]\ge (\mathbb E[X])^2$ for any random variable $X$; In \ccref{proof_c33} we use the fact that $\hat{\delta}_{\tau}, \tau < t$ are $\mathcal{F}_{t-1}$-measurable, and are uncorrelated with $\hat{\delta_t}$; In \ccref{proof_c34} we use the noise assumption; In \ccref{proof_c35} we use the fact that $w_{t-i}$ is $\mathcal{F}_{t-i-1}$-measurable; In  \ccref{proof_c36} we use the fact that $\sqrt{a+b}\le\sqrt a+\sqrt b$ for all $a\ge 0,b\ge 0$. Proof completed.
\end{proof}

\section{Proofs in \Ccref{section_smooth_cvar}}
In this section we prove the main result of \Ccref{section_smooth_cvar} for smoothed CVaR. Recall the expressions
\begin{equation}
    {\psi _\alpha^{\text{smo}} }(t) = \left\{ {\begin{array}{*{20}{l}}
  {t\log t + \frac{{1 - \alpha t}}{\alpha }\log \frac{{1 - \alpha t}}{{1 - \alpha }}}& t\in [0,1/\alpha) \\ 
  +\infty &\text{otherwise} 
\end{array}} \right.
\end{equation}
\begin{equation}
    \psi^{\text{smo},*}_{\alpha}(t) = \frac 1 {\alpha} \log(1-\alpha+\alpha \exp(t)).
\end{equation}

The following proposition shows that $\psi^{\text{smo},*}_{\alpha}$ is Lipschitz-continuous and smooth.

\begin{proposition}
$\psi^{\text{smo},*}_{\alpha}(t)$ is $\frac 1 {\alpha}$-Lipschitz and $\frac 1 {4\alpha}$-smooth.
\end{proposition}

\begin{proof}
We have
\begin{align}
    \left( \psi^{\text{smo},*}_{\alpha} \right)'(t) = \frac{1}{\alpha}\frac{\alpha \exp(t)}{1-\alpha+\alpha \exp(t)} \leq \frac{1}{\alpha},\\
    \left( \psi^{\text{smo},*}_{\alpha} \right)''(t) = \frac{1}{\alpha} \frac{\alpha(1-\alpha)\exp(t)}{(1-\alpha+\alpha \exp(t))^2} \leq \frac{1}{4\alpha}.
\end{align}
where we use $\alpha(1-\alpha) \leq \frac{1}{4}$. Hence the conclusion follows.
\end{proof}

\begin{proposition}
Fix $0<\alpha<1$. When $\lambda \rightarrow 0$, the solution of the DRO problem \ccref{L} for smoothed CVaR tends to the solution for the standard CVaR. 
\end{proposition}
\begin{proof}
For the standard CVaR, the DRO problem can be written as
\begin{equation}
\label{eq:CVar}
    \mathcal L^{\text{CVaR}}(x,\eta):= \lambda \mathbb E_\xi \left[\max\left(\frac {\ell(x;\xi)-\eta} {\alpha\lambda}, 0\right)\right]+\eta=\frac 1 \alpha \mathbb E_\xi\left[ \max\left(\ell(x;\xi)-\eta,0\right)\right]+\eta
\end{equation}
which is irrelevant to $\lambda$. For smoothed CVaR, the DRO problem can be written as
\begin{equation}
\label{eq:smoothCVar}
    \mathcal L^{\text{SCVaR}}_\lambda (x,\eta):= \frac {\lambda}{\alpha} \mathbb E_\xi \left[\log\left(1-\alpha+\alpha\exp\left(\frac{\ell(x;\xi)-\eta}{\lambda}\right)\right)\right]+\eta
\end{equation}
It is easy to see that $\lim_{\lambda\to 0^+} \lambda\log\left(1-\alpha+\alpha\exp\left(\frac z \lambda\right)\right)=\max(z,0)$ for any $z\in\mathbb R$. Therefore \Ccref{eq:smoothCVar} tends to \Ccref{eq:CVar} when $\lambda\to 0^+$.
\end{proof}

\begin{lemma}
\label{smooth_cvar_lemma}
Suppose \Ccref{assumption_general} holds. For smoothed CVaR, the DRO objective \ccref{L} satisfies
\begin{equation}
    \mathbb{E}\| \nabla \widehat{\mathcal{L}}(x,\eta,\xi)\|^2 \leq 2\alpha^{-2}G^2.
\end{equation}
Moreover, $\widehat{\mathcal{L}}(x,\eta)$ is $K$-smooth with $K = \frac{L}{\alpha} + \frac{G^2}{2\lambda\alpha}$.
\end{lemma}

\begin{proof}
We have
\begin{equation}
\begin{aligned}
    \|\nabla_x \widehat{\mathcal{L}}(x,\eta;\xi) \| &= \left(\psi^*\right)'\left( \frac{\ell(x,\xi)-G\eta}{\lambda} \right) \left\| \nabla \ell(x;\xi) \right\| \\
    &\leq \alpha^{-1} \left\| \nabla \ell(x;\xi) \right\| \leq \alpha^{-1}G
\end{aligned}
\notag
\end{equation}
since $\psi^*$ is non-decreasing and $\frac{1}{\alpha}$-Lipschitz continuous.

We also have $\left\|\nabla_{\eta} \widehat{\mathcal{L}}(x,\eta;\xi) \right\| \leq \alpha^{-1}G$. Therefore $\left\| \nabla \widehat{\mathcal{L}}(x,\eta) \right\|^2 \leq 2\alpha^{-2}G^2$.

Now we turn to the smoothness of $\mathcal{L}$. For any $(x,\eta)$ and $(x',\eta')$ we decouple $\nabla \widehat{\mathcal{L}}(x,\eta) - \nabla \widehat{\mathcal{L}}(x',\eta')$ into $A+B$ using the same approach as in \ccref{decouple_smoothness}. Now different from \Ccref{smooth_A}, $A$ can be bounded by 
\begin{equation}
    \|A\|\le \frac L {\alpha} \|x-x'\|
\end{equation}
using the Lipschitz property of $\psi^*$. The bound for $B$ is the same as \Ccref{smooth_B}:
\begin{equation}
    \|B\|\le \frac {G^2}{2\lambda\alpha} \|(x,\eta)^T-(x',\eta')^T\|
\end{equation}
Hence $\mathcal{L}$ is $K$-smooth as desired.
\end{proof}

\begin{theorem}
Suppose that $\psi = \psi_{\alpha}^{\text{smo}}$ and  \Ccref{assumption_general} holds. If we run SGD with properly selected hyper-parameters on the loss $\widehat{\mathcal  L}(x,\eta)$, then the gradient complexity of finding an $\epsilon$-stationary point of $\Psi(x)$ is $\mathcal{O}\left( \alpha^{-3}\lambda^{-1} G^2(G^2+\lambda L)\Delta\epsilon^{-4} \right)$, where $\Delta = \mathcal{L}(x_0,\eta_0) - \inf_{x}\Psi(x)$.
\end{theorem}

\begin{proof}
It is well-known \citep{ghadimi2013stochastic} that the complexity of SGD for finding an $\epsilon$-stationary point is $\mathcal{O}(\Delta K \sigma^2 \epsilon^{-4})$ if the objective function is $K$-smooth and $\sigma^2$ is an upper bound of the variance of stochastic gradients. Now the proof can be completed by using \Ccref{smooth_cvar_lemma}.
\end{proof}

\section{Experiment}
\label{sec_exp_details}
    \subsection{Dataset description}
    \textbf{Imbalanced CIFAR-10.} To demonstrate the effectiveness of our method in DRO-classification setting, we construction an imbalanced classification dataset. The original version of CIFAR-10 contains 50,000 training images and 10,000 validation images of size 32$\times$32 with 10. To create their imbalanced version, we reduce the number of training examples per class and keep the validation set unchanged. We consider the type of random imbalance and use $\rho_{i}$ to denote the sample ratio of $i$th class between the imbalanced and original dataset. $\rho = \{0.804, 0.543, 0.997, 0.593, 0.390, 0.285, 0.959, 0.806, 0.967, 0.660\}$
    \subsection{Implementation details}
    For every training task we jointly tune the parameters learning rate for baseline and our method by grid search and pick the one that achieves the fastest optimization. By default we set momentum = 0.9 for all experiments and $\epsilon = 0.1$ for normalized SGD. We use batch size n = 128 throughout. 
    
    \textbf{Hyper-parameter for $\chi^2$ penalized DRO.}
    In regression setting, we use SGD with lr=0.0002 as our baseline algorithm and set lr=0.005 for normalized SGD. In classification setting, we set lr=0.005 and 0.01 for baseline and our method, respectively.
    
    \textbf{Hyper-parameter for smoothed CVaR.}
    In smooth CVaR, we also divide experiment into two part, regression and classification task. We train CVaR with lr = (0.00005, 0.00005) and smoothed CVaR with lr = (0.001, 0.0001) in regression and classification setting. 
    
\begin{table}[h]
    \centering
    \small
    \caption{Test performance of CVaR-DRO problem for unbalanced CIFAR-10 classification. Each column corresponds to the performance of a particular class. The bolded column indicates the worst-performing class.}
    \setlength\tabcolsep{4pt}
    \begin{tabular}{c|cccccccccc}\hline
        Class & 1 & 2 & 3 & 4 & 5 & \textbf{6} & 7 & 8 & 9 & 10\\ \hline
        Number of training samples & 4020 & 2715 & 4985 & 2965 & 1950 & \textbf{1425} & 4795 & 4030 & 4835 & 3300 \\
        Test acc (CVaR)& 63.0 & 52.6 & 57.9 & 36.2 & 42.1 & \textbf{35.4} & 67.4 & 59.1 & 80.9 & 60.6\\
        Test acc (Smoothed CVaR)& 74.6 & 73.6 & 67.8 & 50.3 & 53.1 & \textbf{37.2} & 80.2 & 79.3 & 90.2 & 67.1\\
        \hline
    \end{tabular}
    \vspace{-10pt}
\end{table}

\end{document}